%% file: arxiv.tex
\documentclass{article}

\makeatletter
\AddToHook{begindocument/before}{%
  \iclrfinalcopy
  \input{arxiv-authors.tex}  \hypersetup{pdftitle={\@title},pdfauthor={\arxivauthornames}}%
  \lhead{Preprint}%
  \renewcommand{\@maketitle}{%
    \vbox{\hsize\textwidth\centering
      {\LARGE\bfseries\@title\par}%
      \vskip 10pt
      \@author\par
      \vskip 10pt
    }%
  }%
}
\makeatother

\usepackage{iclr2027_conference,times}

\input{math_commands.tex}

\usepackage{hyperref}
\usepackage{url}
\usepackage[utf8]{inputenc} %
\usepackage[T1]{fontenc}    %
\usepackage{hyperref}       %
\usepackage{url}            %
\usepackage{booktabs}       %
\usepackage{longtable}
\usepackage{amsfonts}       %
\usepackage{nicefrac}       %
\usepackage{microtype}      %
\usepackage{xcolor}         %
\usepackage{colortbl}

\usepackage{amsmath}
\usepackage{amssymb}
\usepackage{mathtools}
\usepackage{amsthm}
\theoremstyle{plain}
\newtheorem{theorem}{Theorem}[section]
\newtheorem{proposition}[theorem]{Proposition}
\newtheorem{lemma}[theorem]{Lemma}

\theoremstyle{definition}
\newtheorem{definition}[theorem]{Definition}

\theoremstyle{remark}
\newtheorem{remark}[theorem]{Remark}

\usepackage{enumitem}
\usepackage{comment}
\usepackage{amssymb}

\usepackage{graphicx}
\usepackage{amsmath} 
\usepackage{float}
\usepackage{wrapfig}

\usepackage{multirow}
\usepackage{multicol}
\usepackage{makecell}

\usepackage{pifont}
\usepackage{microtype}
\usepackage{graphicx}
\usepackage{booktabs} %

\usepackage[ruled,vlined]{algorithm2e}
\usepackage{algorithmic}

\usepackage{caption}

\usepackage{amsthm}
\theoremstyle{plain}

\newcommand{\0}{{\bf 0}}  %

\title{DraftAttention2: Fast Video Diffusion with Low-Resolution-Guided Mixed-Precision Attention}

\author{
}

\begin{document}
\raggedbottom

\maketitle

\input{sections/0-abstract}

\input{sections/1-introduction}
\input{sections/2-related_work}
\input{sections/3-methodology}
\input{sections/4-experiments}

\input{sections/5-conclusion}

\input{arxiv.bbl}
\input{sections/6-appendix}

\end{document}

%% file: arxiv-authors.tex
\newcommand{\arxivauthornames}{Rui Ding, Haopeng Li, Weize Ma, Yufa Zhou,
  Yitong Li, Xiaoling Zhou, Jiashuo Cao, Liyang Li, Hua Geng, Jiuxiang Gu,
  Jun Lin, Enze Xie, Xuan Shen}

\author{%
  \begin{minipage}[t]{\textwidth}
  \centering\normalfont
  Rui Ding\textsuperscript{1,*,§}\hspace{0.8em}
  Haopeng Li\textsuperscript{2,*,§}\hspace{0.8em}
  Weize Ma\textsuperscript{1,§}\hspace{0.8em}
  Yufa Zhou\textsuperscript{3,§}\hspace{0.8em}
  Yitong Li\textsuperscript{4}\hspace{0.8em}
  Xiaoling Zhou\textsuperscript{5}\\[2pt]
  Jiashuo Cao\textsuperscript{1}\hspace{0.8em}
  Liyang Li\textsuperscript{5}\hspace{0.8em}
  Hua Geng\textsuperscript{5}\hspace{0.8em}
  Jiuxiang Gu\textsuperscript{6,§}\hspace{0.8em}
  Jun Lin\textsuperscript{1}\hspace{0.8em}
  Enze Xie\textsuperscript{4}\hspace{0.8em}
  Xuan Shen\textsuperscript{5,†,§}\\[6pt]
  {\small
  \textsuperscript{1}Nanjing University\quad
  \textsuperscript{2}HKUST(GZ)\quad
  \textsuperscript{3}Duke University\quad
  \textsuperscript{4}NVIDIA\quad
  \textsuperscript{5}Zhejiang University\quad
  \textsuperscript{6}Adobe Research\\[3pt]
  \textsuperscript{*}Equal Contribution,
  \textsuperscript{†}Corresponding Author,
  \textsuperscript{§}Core Contribution
  }
  \end{minipage}%
}

%% file: math_commands.tex
\usepackage{amsmath,amsfonts,bm}

\def\eqref#1{\textup{(\ref{#1})}}

\def\1{\bm{1}}

\DeclareMathAlphabet{\mathsfit}{\encodingdefault}{\sfdefault}{m}{sl}
\SetMathAlphabet{\mathsfit}{bold}{\encodingdefault}{\sfdefault}{bx}{n}

%% file: sections/0-abstract.tex
\begin{abstract}
\vspace{-3mm}

Video generation has broad applications in content creation and entertainment. Diffusion transformers have advanced the quality of generated videos, but attention over spatiotemporal tokens becomes increasingly expensive as video resolution and duration increase. 
We present DraftAttention2, a training-free framework that uses the low-resolution draft attention map to jointly select attention blocks and assign their numerical precision.
Specifically, spatial 2D average- and max-pooled queries and keys capture complementary regional statistics to estimate block importance, and a shared ranking assigns higher precision to important blocks, lower precision to less important retained blocks, and skips the rest under configurable budgets.
Our analysis separates sparsification error from attention-weighted quantization error, establishing when recovering skipped interactions with low-bit computation tightens the output-error bound. This analysis motivates retaining more interactions at low precision while reserving higher precision for blocks with larger attention mass.
To translate these fine-grained assignments into practical speedups, we further develop fused operand preparation and a single attention kernel with precision-specific phases, sharing data movement, softmax statistics, and output accumulation across precisions.
Experiments demonstrate that our method achieves a superior quality–efficiency trade-off over existing efficient video generation methods. 
Notably, its advantage is particularly pronounced for few-step video diffusion, where jointly combining sparsity with 4- and 8-bit mixed-precision computation substantially improves generation quality while retaining significant acceleration.
Code is available at \url{https://github.com/anemoi-project/anemoi}
\vspace{-3mm}

\end{abstract}

%% file: sections/1-introduction.tex
\section{Introduction}
\label{sec:introduction}

Recent advances in video generation have enabled the synthesis of videos with increasingly realistic appearance and coherent motion~\citep{kong2024hunyuanvideo,wan2025}.
These models build on Diffusion Transformers (DiTs)~\citep{dit}, using attention to model spatial and temporal dependencies.
However, extending these capabilities to longer videos and higher resolutions remains computationally expensive.
A major bottleneck is attention over spatiotemporal tokens: its computational cost scales quadratically with sequence length, and the computation is repeated at every denoising step.
Figure~\ref{fig:motivation}(a) illustrates this bottleneck, with attention dominating the DiT runtime.

\begin{wrapfigure}{r}{0.50\textwidth}
\vspace{-5mm}
    \centering
    \includegraphics[width=\linewidth]{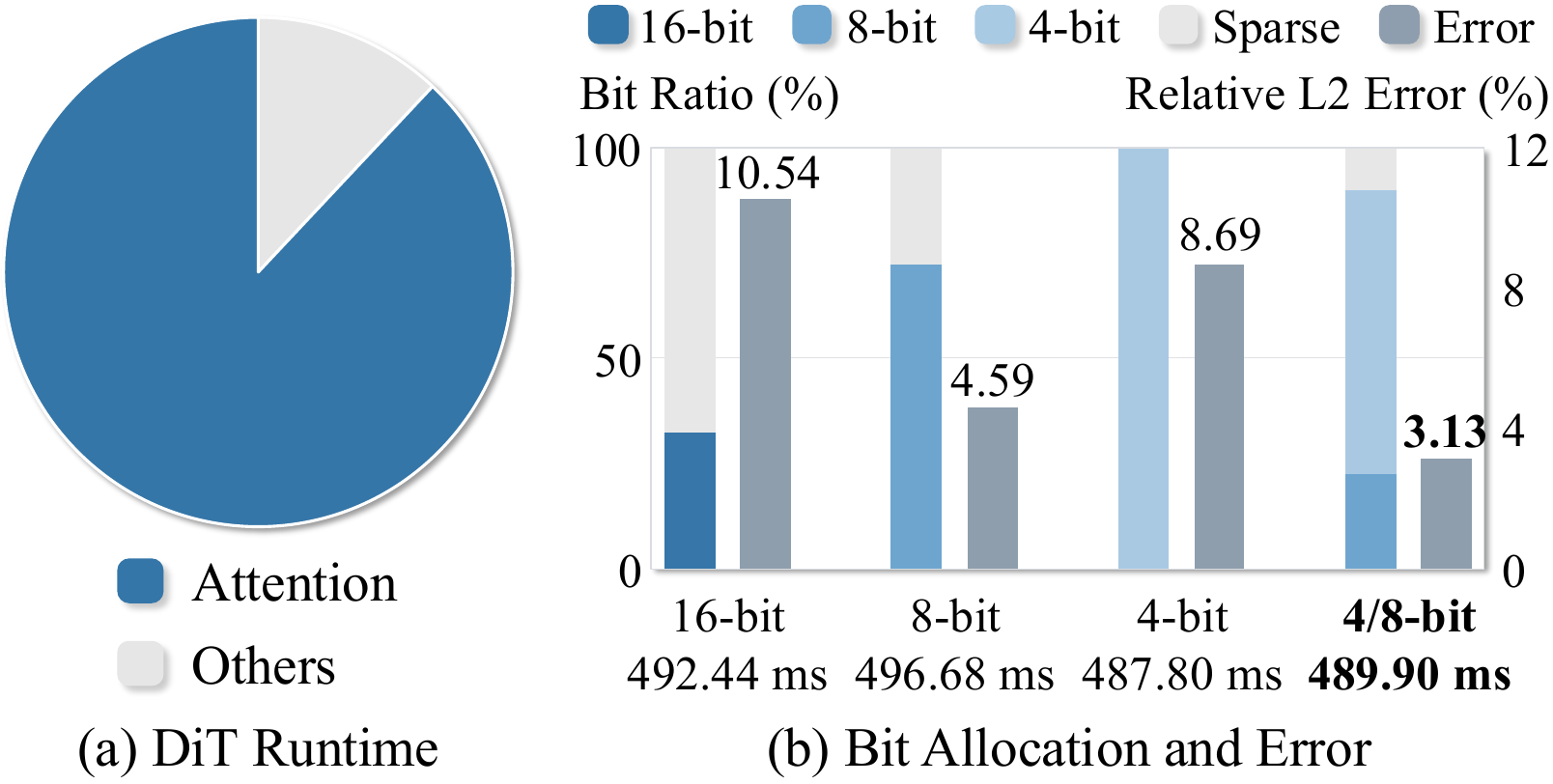}
    \captionsetup{font=small,skip=3pt}
    \caption{Attention cost and precision allocation.
    (a) DiT runtime breakdown.
    (b) Block allocation and relative L2 output error at similar latency.}
    \label{fig:motivation}
    \vspace{-5mm}
\end{wrapfigure}

The sparsity offers a direct way to reduce this cost by retaining a subset of attention interactions as discussed in the previous works~\citep{xi2025sparse_efficient_dit_video_sparseattn,xia2025training,shen2025draftattention}.
Meanwhile, the low-precision computation provides a complementary opportunity: higher matrix-multiplication throughput at lower precisions translates into faster attention.
Low-bit attention exploits this hardware capability through quantization~\citep{zhang2024sageattention2,zhang2025sageattention3,zhang2025sageattention}.
Combining sparsity with quantization enables sparse low-bit attention~\citep{zhang2025spargeattn}, while selective mixed-precision methods further improve accuracy by assigning higher precision to important interaction blocks~\citep{sharratt2026thriftattention}.

Building on block-sparse attention, we generalize the binary retain-or-skip decision into a joint sparsity–precision allocation, where each block is skipped or computed at a selected numerical precision.
Low-precision computation replaces part of the sparsification, while reserving a small high-precision budget for the most important blocks.
Figure~\ref{fig:motivation}(b) illustrates this trade-off: at similar measured latency, mixed 4/8-bit attention achieves lower output error than the single-precision sparse and dense allocations.
This design introduces two key challenges: efficiently estimating block importance to guide both block selection and precision assignment, and executing heterogeneous precisions without incurring substantial switching or output-merging overhead.

\begin{figure}[t]
 \centering
 \includegraphics[width=\linewidth]{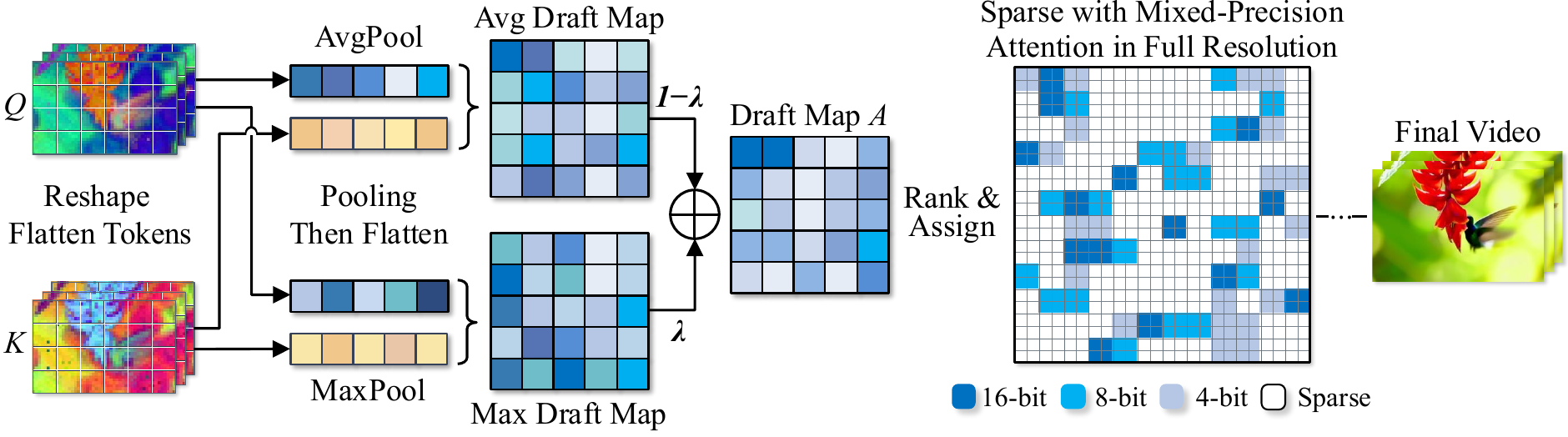}
 \vspace{-6mm}
 \caption{
 Overview of DraftAttention2. Average- and max-pooled queries and keys produce draft attention maps, whose weighted mixture guides block selection and precision allocation. 
 }
 \label{fig:pipeline-overview}
 \vspace{-4mm}
\end{figure}

\textbf{Framework overview.}
We present \textbf{DraftAttention2}, a training-free framework for sparse and mixed-precision attention in video diffusion (Figure~\ref{fig:pipeline-overview}).
The key idea is to construct a low-resolution draft attention map from complementary average- and max-pooled queries and keys, and use it to jointly guide block selection and precision allocation. 
The resulting regional importance scores are ranked across all region pairs within each attention head: higher-scoring blocks are assigned higher precision, lower-ranked retained blocks use lower precision, and the rest are skipped under configurable budgets.
This shared ranking allows both the number of retained blocks and their precision assignments to adapt across query regions under a common head-level budget. 
Finally, spatial token reordering maps these regional routing decisions to contiguous attention blocks for efficient full-resolution computation.

\textbf{Theoretical analysis.}
Our analysis connects joint sparsity--precision allocation to attention-output error.
Average pooling preserves mean regional logits, while the max pooling supplies variation information beyond the mean, motivating the combined draft.
We then decompose the output error into sparsification and attention-weighted quantization terms. 
This decomposition shows that computing moderately important interactions at low precision can be a particularly favorable replacement for sparsification: recovering these interactions with low-bit computation can substantially tighten the output-error bound relative to the original DraftAttention formulation. 
Finally, under fixed precision budgets, we show that assigning higher precision to interactions with larger attention mass minimizes the resulting error bound, providing theoretical justification for the draft-guided precision allocation.

\textbf{Low-overhead mixed-precision execution.}
To efficiently execute the resulting fine-grained precision assignments, we fuse spatial reordering, pooling, and quantization into the preparation of draft descriptors and full-resolution attention operands. 
Retained blocks are grouped by precision and processed in consecutive phases within a single attention kernel, so the number of precision transitions depends only on the number of active precision formats rather than the number of blocks. 
All phases share online softmax statistics and a common output accumulator, with lightweight rescaling at phase boundaries. 
This design avoids separate per-precision outputs and an additional merging kernel, substantially reducing the overhead of mixed-precision execution.

\textbf{Experimental results.}
We evaluate DraftAttention2 on HunyuanVideo and MiniMax-H3, measuring video generation quality, similarity to dense-attention outputs, and DiT inference speed.
The results show that draft-guided sparse and mixed-precision attention accelerates video generation with competitive quality.
Allocating a small fraction of blocks to higher precision improves output similarity while preserving most of the speed benefit of low-precision computation, making precision allocation an effective way to control the quality--speed trade-off.
Notably, on few-step video diffusion, where approximation errors are less likely to be corrected over repeated denoising steps, DraftAttention2 exhibits a substantially stronger advantage over existing efficient attention methods, demonstrating its effectiveness in this more practical generation setting.

Our contributions are summarized as follows:
\begin{itemize}[leftmargin=*]
\vspace{-2.5mm}
    \item We introduce DraftAttention2, which uses low-resolution guidance to jointly allocate sparsity and numerical precision across attention blocks under configurable budgets.
    Our analysis shows that computing moderately important interactions at low precision offers a particularly favorable error--efficiency trade-off: replacing their sparsification with low-bit computation can substantially tighten the output-error bound, leading to an overall tighter error bound than the previous DraftAttention.
    
    \item We develop an execution scheme that supports multiple precisions within the same attention computation with low overhead.
    Fused preparation shares data movement between draft construction and operand quantization, while precision-specific phases share normalization and output accumulation.
    
    \item Experiments on HunyuanVideo and MiniMax-H3 demonstrate faster DiT inference with competitive generation quality and show that a small high-precision budget improves fidelity to dense-attention outputs.
    DraftAttention2 further delivers particularly strong gains on few-step video diffusion, substantially improving generation quality over existing efficient attention methods while maintaining significant speedup.
    
\end{itemize}

%% file: sections/2-related_work.tex
\section{Related Work}
\label{sec:related_work}

\noindent\textbf{Sparse attention for video generation.}
Video diffusion models exhibit local spatiotemporal structure and data-dependent attention sparsity.
Sparse VideoGen~\citep{xi2025sparse_efficient_dit_video_sparseattn} uses online profiling to select spatial or temporal attention patterns for each head.
DraftAttention~\citep{shen2025draftattention} uses spatially average-pooled queries and keys to select blocks for full-resolution sparse attention.
AdaSpa~\citep{xia2025training} searches for block-sparse masks online and reuses them between designated search steps.
Sparse VideoGen2~\citep{svg2} clusters and reorders query and key tokens, using cluster-centroid attention estimates for top-$p$ selection.

\noindent\textbf{Low-bit and mixed-precision attention.}
SageAttention2~\citep{zhang2024sageattention2} combines per-thread INT4 quantization of queries and keys with FP8 quantization of attention weights and values.
SageAttention3~\citep{zhang2025sageattention3} applies NVFP4 quantization to both attention matrix multiplications, with two-level scaling of the attention weights.
SpargeAttention~\citep{zhang2025spargeattn} combines quantization with block selection that skips both matrix multiplications and an online softmax-aware filter that further skips value accumulation.
Block-level mixed-precision methods allocate higher precision to selected interactions.
Diagonal-Tiled Mixed-Precision Attention (DMA)~\citep{ding2026dma} retains higher-precision query--key computation near the diagonal.
BMAttn~\mbox{\citep{anonymous2026bmattn}} and ThriftAttention~\citep{sharratt2026thriftattention} both adopt block-level mixed precision, with BMAttn partitioning attention into high-precision, low-precision, and skipped regions using offline-calibrated per-head boundaries, while ThriftAttention uses pooled queries and keys to assign FP16 to top-$k$ key blocks and FP4 to the remaining unmasked blocks.
Our method jointly selects attention blocks and assigns them multiple precisions using a shared low-resolution importance ranking over all blocks in each attention head.
This allows both the number of retained blocks and their precision allocation to vary across query regions under a common per-head budget.

%% file: sections/3-methodology.tex
\section{Methodology}
\label{sec:methodology}

As shown in Figure~\ref{fig:pipeline-overview}, DraftAttention2 extends block selection to a joint allocation of sparsity and precision.
A low-resolution draft attention map ranks regional interactions, assigning higher precision to important blocks, lower precision to other retained blocks, and skipping the remainder.
We first describe this allocation, then relate it to attention-weighted quantization error, and finally show how to execute the resulting mixed-precision attention with shared operand preparation and normalization.

\subsection{Draft-Guided Precision Allocation}
\label{sec:precision-allocation}

For one attention head, let $Q,K,V\in\mathbb R^{N\times d}$, with dense attention $P=\operatorname{softmax}_{\rm row}(QK^\top/\sqrt d)$ and output $O^{\rm dense}=PV$.
Following DraftAttention~\citep{shen2025draftattention}, we estimate regional importance before evaluating token-level attention.

\paragraph{Spatial grouping and draft attention.}
Spatially neighboring video tokens provide a natural unit for pooling and block-wise computation.
We partition each frame into connected regions $\{\mathcal R_a\}$, with balanced token counts bounded by the attention block capacity.
Section~\ref{sec:hardware} describes the spatial partition and its alignment with attention blocks.

For each region, we form an average descriptor and a channel-wise maximum descriptor:
\begin{equation}
 \bar Q_a=\frac{1}{|\mathcal R_a|}\sum_{u\in\mathcal R_a}Q_u,
 \qquad Q^{\max}_{a,c}=\max_{u\in\mathcal R_a}Q_{u,c},
 \label{eq:v2-descriptors}
\end{equation}
with $\bar K$ and $K^{\max}$ defined analogously.
We compute a draft map from each pair of pooled query and key tensors, normalize each map over key regions, and combine them to obtain the block importance scores.

\begin{definition}[Draft-Guided Precision Allocation]
\label{def:v2-allocation}
Let $\bar Q,\bar K$ and $Q^{\max},K^{\max}$ be the regional average and channel-wise maximum descriptors.
For a pooling weight $\lambda\in[0,1]$, define
\begin{equation}
 A=(1-\lambda)\operatorname{softmax}_{\rm row}
 \!\left(\frac{\bar Q\bar K^\top}{\sqrt d}\right)
 +\lambda\operatorname{softmax}_{\rm row}
 \!\left(\frac{Q^{\max}(K^{\max})^\top}{\sqrt d}\right).
 \label{eq:v2-draft}
\end{equation}
Given head-level precision quotas $T_{16},T_8,T_4$, rank video-region pairs by $A_{ab}$.
Assign the first $T_{16}$ pairs to the 16-bit path, the next $T_8$ to INT8, and the next $T_4$ to NVFP4; skip the rest.
\end{definition}

\begin{remark}[Complementary Regional Statistics]
\label{rem:v2-pooling}
By bilinearity, $\bar Q_a^\top\bar K_b/\sqrt d$ equals the mean token-pair logit between regions $\mathcal R_a$ and $\mathcal R_b$.
Regions with the same average descriptor can nevertheless have different channel-wise maxima.
The Max branch exposes these variations to the draft, while $\lambda$ controls their contribution; $\lambda=0$ recovers average-only guidance.
\end{remark}

\begingroup
\IncMargin{1em}
\begin{algorithm}[H]
 \small
 \caption{Draft-guided sparsity and precision allocation}
 \label{alg:v2-forward}
 \DontPrintSemicolon
 \LinesNumbered
 \SetAlgoNlRelativeSize{-1}
 \KwIn{$Q,K$; regions $\{\mathcal R_a\}_{a=1}^{R}$; $\lambda$; quotas $T_{16},T_8,T_4$}
 \KwOut{Precision assignments $p\in\{0,4,8,16\}^{R\times R}$}
 $\bar Q,\bar K\leftarrow\operatorname{AvgPool}(Q,\{\mathcal R_a\}),\operatorname{AvgPool}(K,\{\mathcal R_a\})$\;
 $Q^{\max},K^{\max}\leftarrow\operatorname{MaxPool}(Q,\{\mathcal R_a\}),\operatorname{MaxPool}(K,\{\mathcal R_a\})$\;
 $A\leftarrow(1-\lambda)\operatorname{softmax}_{\rm row}(\bar Q\bar K^\top/\sqrt d)
 +\lambda\operatorname{softmax}_{\rm row}(Q^{\max}(K^{\max})^\top/\sqrt d)$\;
 $I\leftarrow\operatorname{argsort}(\operatorname{vec}(A),\ \mathrm{descending}=\mathrm{True},\ \mathrm{stable}=\mathrm{True})$\;
 $T\leftarrow T_{16}+T_8+T_4$; $p\leftarrow\mathbf{0}_{R^2}$\;
 $p[I[\,{:}\,T_{16}]]\leftarrow16$\;
 $p[I[T_{16}\,{:}\,T_{16}+T_8]]\leftarrow8$\;
 $p[I[T_{16}+T_8\,{:}\,T]]\leftarrow4$\;
 \Return{$\operatorname{reshape}(p,R,R)$}\;
\end{algorithm}
\endgroup

\paragraph{Joint allocation of sparsity and precision.}
For $R$ video regions, the draft contains $R^2$ scores, replacing token-level scoring for the allocation decision.
The total retained budget is $T=T_{16}+T_8+T_4$, giving a retained fraction of $T/R^2$.
The individual quotas specify how many blocks use each precision.
Since ranking is performed across the whole map within each head, query regions can receive different numbers of retained blocks and different precision allocations under the same head-level budget.

Each selected region pair is then evaluated at full resolution using its assigned precision.
All retained blocks for a query contribute to a common softmax normalization and output, including blocks assigned to different precisions.
The pooled descriptors guide this computation; the output is formed from the full-resolution values.
Algorithm~\ref{alg:v2-forward} gives the draft-guided allocation, where $p_{ab}\in\{0,4,8,16\}$ specifies the precision of region pair $(a,b)$ and $0$ denotes a skipped pair.
Section~\ref{sec:hardware} describes attention computation with these assignments.
Routing details and numerical formats are given in Appendix~\ref{app:v2-implementation}.

\subsection{Quantization and Attention Importance}
\label{sec:quantization-analysis}

DraftAttention selects which interactions to compute at the reference precision.
DraftAttention2 also chooses how accurately to compute them: a skipped block can instead contribute through a low-bit path.
We analyze when the recovered attention mass outweighs the additional quantization error, and how precision should be distributed among retained blocks.

\paragraph{Illustrative example.}
Figure~\ref{fig:motivation}(b) compares four allocations on the same captured video-attention input, using average-only draft guidance ($\lambda=0$).
At similar measured latency, sparse 16-bit attention loses information through omitted interactions, whereas dense 4-bit attention retains every interaction but incurs larger quantization error.
The mixed 4/8-bit allocation balances these effects, reducing relative output error to $3.132\%$.
The 16-bit label denotes the reference-precision path; the analysis isolates the additional error from low-bit quantization.

\paragraph{From sparsification to quantization error.}
Quantization perturbs both the attention logits and the weighted value aggregation.
We take the original BF16 inputs as the reference and leave the 16-bit path unquantized in the reconstruction model; ordinary floating-point roundoff is treated separately in Appendix~\ref{app:v2-error-proof}.
For a query $u$ with nonempty retained support $\Omega_u$, let $P^\Omega_{uv}$ be the exact attention probability renormalized over $\Omega_u$.
Let $\Omega_u^{(p)}$ contain the retained keys assigned precision $p$, including any retained prefix interactions.
Define
\begin{equation}
 m_u^{(p)}=\sum_{v\in\Omega_u^{(p)}}P^\Omega_{uv},
 \qquad
 \delta_u=\sum_{v\notin\Omega_u}P_{uv},
 \qquad
 \sum_{p\in\{4,8,16\}}m_u^{(p)}=1.
 \label{eq:v2-block-mass}
\end{equation}
Here $\delta_u$ measures omitted dense attention mass, while $m_u^{(p)}$ measures the share of retained attention assigned to each precision.

\begin{lemma}[Attention-Weighted Quantization Error]
\label{thm:v2-quantization}
Suppose $\|V_v\|_2\le V_{\max}$ and the reconstructed operands satisfy the error bounds in Appendix~\ref{app:v2-error-model}.
Let $\varepsilon_p$ be the uniform error envelope for low-bit precision $p\in\{4,8\}$ in Equation~\eqref{eq:v2-format-envelope}.
The reconstructed mixed-precision output satisfies
\begin{equation}
 \|O_u^{\rm dense}-\widehat O_u\|_2
 \le
 \underbrace{2V_{\max}\delta_u}_{\text{sparsification}}
 +
 \underbrace{m_u^{(8)}\varepsilon_8
 +m_u^{(4)}\varepsilon_4}_{\text{quantization}}.
 \label{eq:v2-total-error}
\end{equation}
\end{lemma}

The quantization term weights each precision's error by the attention mass computed at that precision.
The envelopes account jointly for logit, value, and probability-weight reconstruction, consistent with the shared normalization used by the kernel.
The two terms expose the choice between omitting an interaction and retaining a quantized approximation to it.

\begin{proposition}[From Sparse to Mixed-Precision Attention]
\label{prop:v2-recovery}
Compare a DraftAttention (v1) sparse allocation with a DraftAttention2 allocation that preserves the same 16-bit blocks and computes additional blocks at 8 or 4 bits.
For one query, let $\delta_{\rm v2}$ be the dense attention mass skipped by v2 and $\tau_p$ the dense mass assigned to its $p$-bit path, $p\in\{4,8\}$.
Assume $0<\tau=\tau_8+\tau_4<1-\delta_{\rm v2}$, so the shared 16-bit support is nonempty.
Lemma~\ref{thm:v2-quantization} gives the respective error bounds
\begin{equation}
 B_{\rm v1}=2V_{\max}(\delta_{\rm v2}+\tau),\qquad
 B_{\rm v2}=2V_{\max}\delta_{\rm v2}+
 \frac{\tau_8\varepsilon_8+\tau_4\varepsilon_4}{1-\delta_{\rm v2}}.
 \label{eq:v2-recovery-bounds}
\end{equation}
In particular, $B_{\rm v2}<B_{\rm v1}$ whenever
\begin{equation}
 \tau_8\varepsilon_8+\tau_4\varepsilon_4
 <2V_{\max}\tau(1-\delta_{\rm v2}).
 \label{eq:v2-recovery-condition}
\end{equation}
\end{proposition}

\begin{remark}[Interpreting the Comparison]
\label{rem:v2-recovery-interpretation}
The two allocations preserve the same reference-precision blocks but have different omitted attention masses:
$\delta_{\rm v1}=\delta_{\rm v2}+\tau_4+\tau_8$.
Low-bit recovery reduces the sparsification bound by $2V_{\max}\tau$ and introduces the quantization term in Equation~\eqref{eq:v2-recovery-bounds}; the bound improves when the latter is smaller than the former.
For these matched allocations, the skipped fractions of video-region pairs satisfy $\rho_{\rm v1}=\rho_{\rm v2}+\rho_4+\rho_8$, where $\rho_p$ is the fraction assigned to precision $p$ by v2.
These block fractions differ from the attention masses above.
The comparison motivates computing moderately important interactions at low precision while leaving the least important interactions sparse.
Figure~\ref{fig:motivation}(b) illustrates the resulting trade-off at similar latency.
\end{remark}

\paragraph{Allocating precision within the retained support.}
Once the support is fixed, the same decomposition determines where higher precision is most useful.
For a retained region pair $(a,b)$, define its query-averaged mass
$W_{ab}=N_{\rm vid}^{-1}\sum_{u\in\mathcal R_a}\sum_{v\in\mathcal R_b}P^\Omega_{uv}$,
where $N_{\rm vid}$ is the number of video queries.

\begin{theorem}[Importance-Guided Precision Allocation]
\label{prop:v2-allocation}
Fix the retained support and precision quotas.
If $0\le\varepsilon_8\le\varepsilon_4$, then the assignment-dependent error bound
$\sum_{(a,b):\,p_{ab}\in\{4,8\}}W_{ab}\varepsilon_{p_{ab}}$
is minimized by assigning the 16-bit path to the $T_{16}$ largest masses, INT8 to the next $T_8$, and NVFP4 to the remainder.
\end{theorem}

The theorem motivates the precision ordering in Definition~\ref{def:v2-allocation}: larger attention mass makes an interaction more sensitive to quantization.
Computing the exact masses requires full-resolution attention, so the router uses draft scores $A_{ab}$ as low-cost proxies.
Together, the results motivate retaining more interactions at low precision while reserving higher precision for the most important blocks.
Proofs and the reconstruction-error model are in Appendix~\ref{app:v2-quantization}.

\subsection{Hardware-Efficient Execution}
\label{sec:hardware}

To efficiently execute the joint sparsity--precision allocation, spatial regions must align with contiguous attention blocks, while blocks assigned different precisions must share normalization and output accumulation.
We address these requirements through fused operand preparation and phased mixed-precision execution.

\begin{figure}[t]
 \centering
 \includegraphics[width=\linewidth]{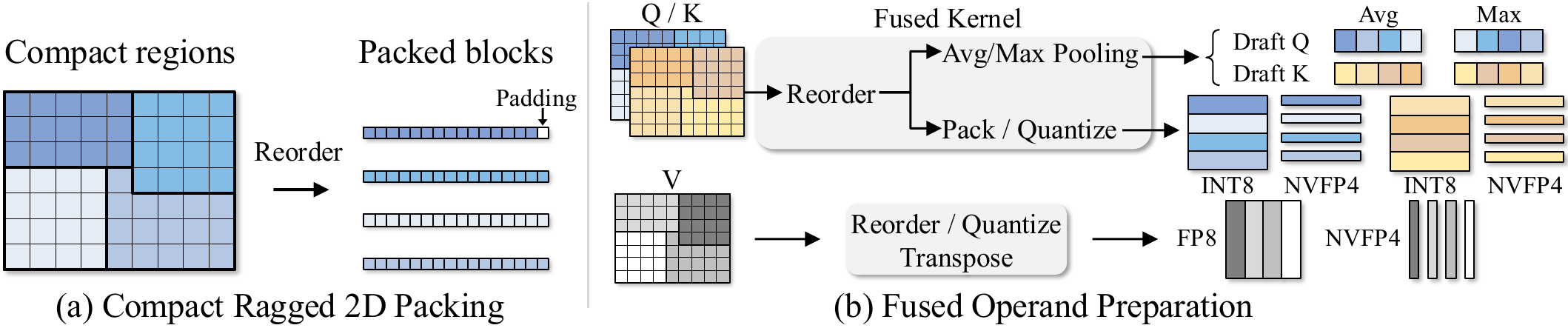}
 \vspace{-6mm}
 \caption{Compact spatial packing and fused operand preparation. (a) Compact spatial regions are packed into contiguous blocks. (b) Q/K reordering, pooling, and quantization share a kernel; V reordering and transposition are fused with operand packing and loading. Colors denote regions; gray cells denote padding.}
 \label{fig:fused-preparation}
 \vspace{-4mm}
\end{figure}

\paragraph{Fused quantization and reordering.}
Tokens belonging to one spatial region can be scattered in raster order, whereas attention kernels operate on fixed-size contiguous blocks.
We construct a ragged 2D partition with balanced token counts across regions, prioritizing the smallest total perimeter among candidate partitions (Figure~\ref{fig:fused-preparation}).
These criteria maximize block utilization while preserving 2D spatial locality.
The partition adapts region shapes to the frame dimensions and uses the minimum number of blocks, with less than one block of total padding per frame.
We pack each region into a contiguous block, using the same spatial indices for pooling and operand preparation.
The partition is reused across layers and denoising steps.

Computing pooled descriptors and preparing each numerical format separately would repeatedly move the same token features through global memory.
Our Q/K preparation kernel fuses spatial gathering, average- and max-pooling, and quantization, producing the draft descriptors together with packed operands and quantization scales.
The value path similarly fuses spatial reordering with operand preparation.
To reconcile the $QK^\top$ accumulator layout with the operand layout of the probability--value multiplication~\citep{shah2024flashattention3}, we align the intra-block $K/V$ layouts so that the attention weights can be converted and packed within each thread after softmax.
Value transposition is incorporated into operand packing and loading.
Thus, the data movement needed to prepare attention operands also supplies the draft descriptors, without standalone spatial reordering, pooling, or value-transposition kernels.

\paragraph{Phased execution.}
To limit precision switching, we group retained key blocks for each query block into NVFP4, INT8, and 16-bit segments.
A single attention kernel processes the nonempty segments in this order, using the corresponding operand-loading and matrix-multiplication path.
If $n_{\mathrm{phase}}$ phases are nonempty, the number of precision transitions is
\begin{equation}
 n_{\mathrm{switch}}=\max\!\left(n_{\mathrm{phase}}-1,0\right)\le 2.
 \label{eq:precision-switches}
\end{equation}
The number of precision transitions therefore depends only on the active formats, independently of the number of retained blocks.
Consecutive phases also reuse a shared-memory workspace, with operand fragments and quantization temporaries local to each phase.

\paragraph{Cross-precision accumulation.}
To accumulate contributions from different precisions into a common attention output, all phases share FP32 row maxima, exponential sums, and output accumulators using online softmax~\citep{dao2023flashattention2}.
Let $\widetilde P$ denote the unnormalized exponential weights, and let $U_c$ and $\ell$ be the output numerator in channel $c$ and the softmax denominator.
We rescale the accumulators at phase boundaries to account for the probability and value quantization scales.
For example, the INT8 path uses E4M3 operands for $\widetilde PV$.
With probability dequantization scale $\eta$ and nonzero per-channel value dequantization scale $\delta_c^V$, its entry and exit transformations are
\begin{equation}
 \begin{aligned}
 \text{Enter:}\quad
 U_c^{\mathrm{enc}}&=\frac{U_c}{\eta\delta_c^V},
 &\ell^{\mathrm{enc}}&=\frac{\ell}{\eta}, \quad \quad
 \text{Exit:}\quad
 U_c&=\eta\delta_c^V U_c^{\mathrm{enc}},
 &\ell&=\eta\ell^{\mathrm{enc}}.
 \end{aligned}
 \label{eq:precision-state-conversion}
\end{equation}
The normalized output satisfies $U_c/\ell=\delta_c^V U_c^{\mathrm{enc}}/\ell^{\mathrm{enc}}$, so each phase can update the shared output accumulator after rescaling.
The denominator accumulates exponential weights before their conversion to matrix operands, as in the reconstruction model of Section~\ref{sec:quantization-analysis}.
We normalize once after all retained blocks, avoiding per-precision output writes and a separate merging kernel.
Appendix~\ref{app:v2-implementation} details the spatial layout, quantization, and phase-specific scaling.

%% file: sections/4-experiments.tex
\section{Experiments}
\label{sec:experiments}
\setlength{\intextsep}{6pt}
\renewcommand{\floatpagefraction}{0.9}

\subsection{Experimental Setup}
\label{sec:experimental-setup}

\paragraph{Model family.}
We evaluate DraftAttention2 on HunyuanVideo-13B~\citep{kong2024hunyuanvideo} and MiniMax-H3~\citep{minimaxai2026minimaxh3},
including 4-step and 8-step MiniMax-H3 variants using FL2VA Turbo LoRAs~\citep{modeltc2026minimaxh3turbo}.
We generate 129-frame videos at $1280\times720$ resolution with HunyuanVideo and 362-frame videos at $1344\times768$ resolution with MiniMax-H3 and its few-step variants.
We retain dense attention for the first 10 denoising steps of the standard models and the first quarter of steps in the few-step variants.
We compare with dense attention, XAttention~\citep{xu2025xattention_efficient_sparseattn}, SVG2~\citep{svg2}, SVG-EAR~\citep{zhou2026svgear}, SVOO~\citep{luo2026svoo}, Sol-Attn~\citep{li2026solattn}, SpargeAttention~\citep{zhang2025spargeattn}, and SageAttention3~\citep{zhang2025sageattention3}.
For clustering-based baselines, we select the numbers of query and key clusters to obtain comparable inference speedups.

\paragraph{Metrics and prompts.}
We evaluate the quality of generated videos with VBench~\citep{huang2023vbench} and their similarity to dense-attention outputs with PSNR, SSIM, and LPIPS~\citep{zhang2018perceptual}.
We report eight VBench dimensions and their arithmetic mean.
For similarity evaluation, we use dense references generated with matching prompts and random seeds.
All videos are generated using prompts from VBench, and latency is measured on an NVIDIA RTX 5090 GPU.

\input{tables/main-results}
\input{tables/fewstep-results}

\subsection{Main Results}
\label{sec:main-results}

\paragraph{Higher performance on recent video diffusion models.}
Table~\ref{tab:main-results} compares generation quality and DiT speedup on HunyuanVideo and MiniMax-H3.
Compared with sparse-attention methods including XAttention, SVG2, Sol-Attn, SVG-EAR, and SVOO, DraftAttention2 achieves higher inference speed while maintaining competitive generation quality and fidelity to dense attention.
For example, our sparse configurations reach up to $2.38\times$ and $2.33\times$ speedup on HunyuanVideo and MiniMax-H3, respectively, exceeding the speedups of these sparse-only baselines.
We further introduce mixed-precision allocation to improve quality at higher acceleration.
At speedups matched to SpargeAttention, DraftAttention2 consistently achieves better generation quality and dense-output fidelity on both HunyuanVideo and MiniMax-H3.
We additionally compare against SageAttention3 under matched speedups, where DraftAttention2 achieves better generation quality with comparable acceleration.
These results demonstrate that jointly allocating sparsity and numerical precision provides a better quality--efficiency trade-off than sparse-only or uniform low-precision attention.

\paragraph{Strong performance on few-step generation.}
Table~\ref{tab:fewstep-results} evaluates DraftAttention2 on 4- and 8-step MiniMax-H3 generation, where aggressive attention approximation is more challenging because fewer denoising steps are available to correct accumulated errors.
Across both settings, DraftAttention2 achieves a stronger quality--efficiency trade-off than sparse-only baselines by replacing part of the sparsification with low-bit computation.
For 4-step generation, our mixed-precision configurations closely preserve dense quality while substantially improving fidelity over sparse baselines at similar acceleration.
For 8-step generation, DraftAttention2 exhibits the same trend, consistently improving generation quality and dense-output fidelity over sparse-only baselines while maintaining comparable or higher acceleration.
At speedups comparable to SpargeAttention, mixed 4/8-bit allocation further improves dense-output fidelity.
For SageAttention3, DraftAttention2 achieves better generation quality at comparable speedups, further demonstrating the advantage of selective sparsity--precision allocation over uniform low-bit attention.
These results highlight the particular benefit of joint sparsity and mixed-precision allocation for practical few-step video generation, where preserving moderately important interactions becomes increasingly critical.

\paragraph{Visualization.}
Figure~\ref{fig:qualitative-comparison} compares DraftAttention2 with competing methods on few-step video generation.
As highlighted by the red boxes, DraftAttention2 preserves the rail junction and the cat's appearance, producing results close to the dense baseline.

\begin{figure}[!t]
\centering
\includegraphics[width=\linewidth]{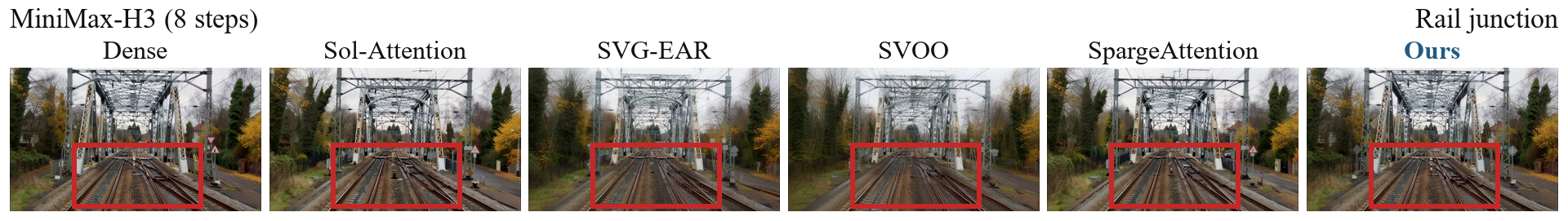}\par
\vspace{4pt}
\includegraphics[width=\linewidth]{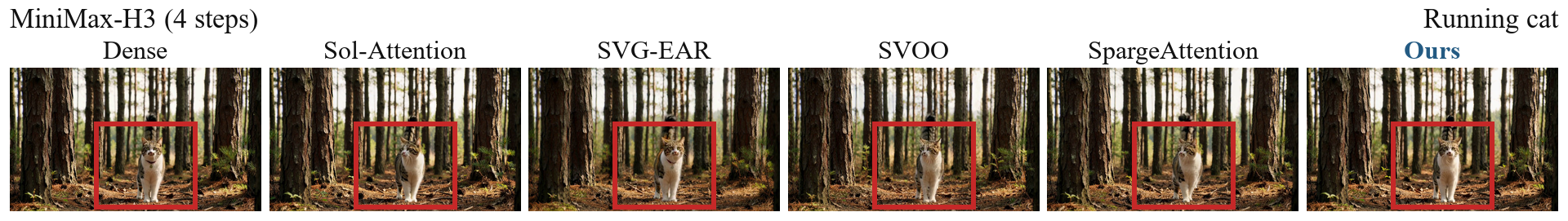}
\captionsetup{font=small,skip=3pt}
\caption{Visual comparison on 8-step (top) and 4-step (bottom) diffusion with MiniMax-H3 FL2VA Turbo LoRAs. Red boxes highlight structural details.}
\label{fig:qualitative-comparison}
\end{figure}

\subsection{Ablation Study}
\label{sec:ablation}

\begin{figure}[!t]
\begin{minipage}[t]{0.48\textwidth}
\vspace{0pt}
\input{tables/maxpool-ablation}
\end{minipage}\hfill
\begin{minipage}[t]{0.5\textwidth}
\vspace{0pt}
\centering
\includegraphics[width=\linewidth]{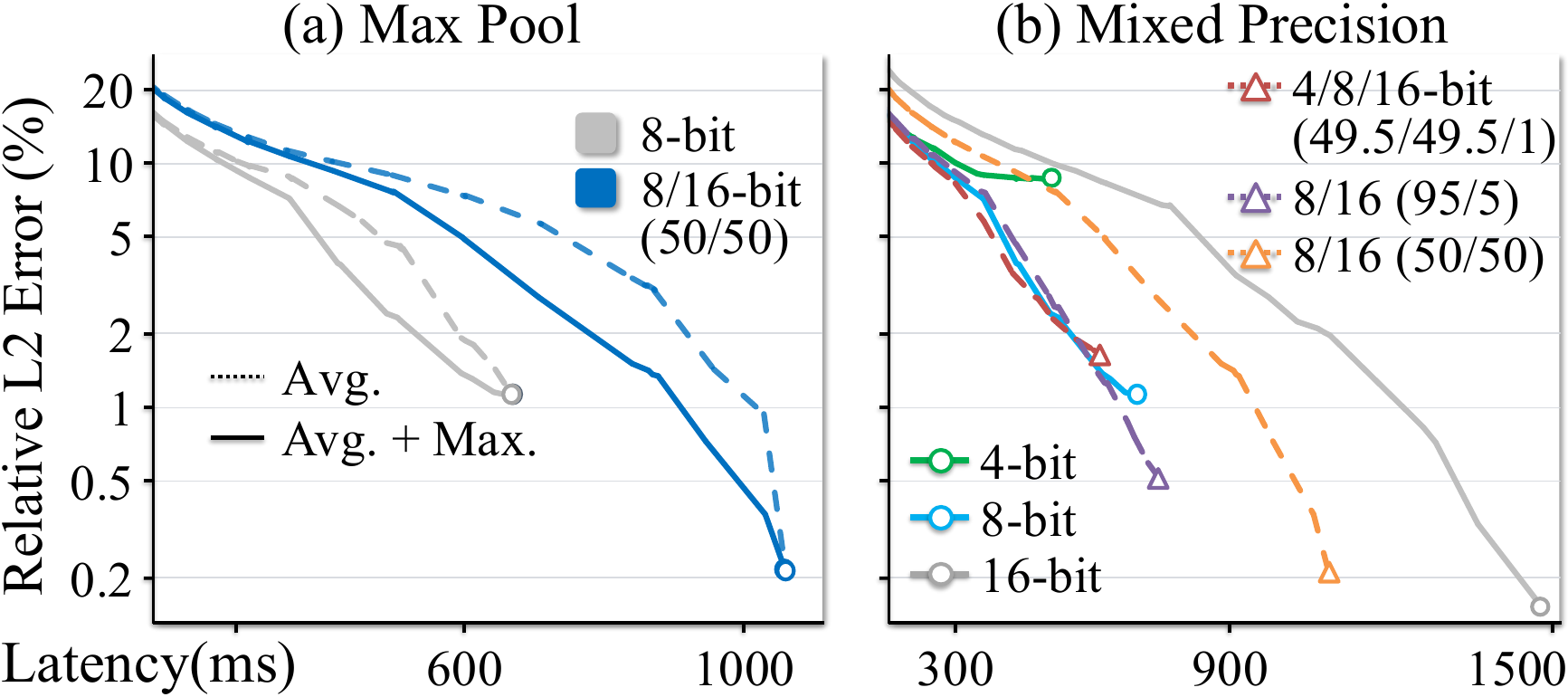}
\captionsetup{font=small,skip=3pt}
\captionof{figure}{Attention output error versus latency. (a) Average pooling with and without MaxPool. (b) Single- and mixed-precision attention with MaxPool. Curves vary sparsity at fixed precision ratios; open circles mark zero sparsity. Ratios are percentages of retained connections.}
\label{fig:operator-tradeoff}
\end{minipage}
\end{figure}

We study pooling and precision allocation on a dense 4-step trajectory capture with $C=128$.
We measure relative L2 error against the dense output and complete attention-operator latency, including operand preparation and routing.

\paragraph{Effect of MaxPool.}
Figure~\ref{fig:operator-tradeoff}(a) compares average pooling with the average--max mixture ($\lambda=0.2$).
Adding MaxPool improves the error--latency trade-off for both 8-bit and mixed 8/16-bit attention.
The improvement at a fixed block budget demonstrates the value of the complementary max-pooled draft for guiding block selection.
Table~\ref{tab:maxpool-ablation} further shows consistent improvements in video generation performance at fixed sparsity and precision budgets.
The improvements extend to both 16-bit and mixed 8/16-bit attention, showing that the additional guidance benefits different precision allocations.

\paragraph{Effect of mixed precision.}
Figure~\ref{fig:operator-tradeoff}(b) shows the benefit of a small high-precision budget.
Across a range of latency budgets, introducing a small fraction of higher-precision computation consistently lowers attention output error compared with single-precision alternatives.
Selective high-precision computation thus reduces error beyond the zero-sparsity endpoint of pure 8-bit attention.
Reducing sparsity alone cannot remove the quantization error that remains when every block is computed at low precision.
Draft-guided precision allocation addresses this error by reserving higher precision for important interactions, improving accuracy within the measured attention latency budget.

%% file: tables/main-results.tex
\begin{table}[!t]
\centering
\captionsetup{font=small,skip=3pt}
\caption{Video generation quality and DiT speedup. VBench Avg is reported as a percentage. Bit allocations give percentages of 0-, 4-, 8-, and 16-bit attention blocks, where 0-bit denotes sparse and 16-bit denotes dense.}
\label{tab:main-results}
\fontsize{7.5}{9}\selectfont
\setlength{\tabcolsep}{0.9pt}
\renewcommand{\arraystretch}{1.00}
\setlength{\heavyrulewidth}{0.6pt}
\setlength{\lightrulewidth}{0.3pt}
\arrayrulecolor{black!65}
\definecolor{oursrow}{RGB}{224,237,250}
\resizebox{\linewidth}{!}{%
\begin{tabular}{@{}c@{\hspace{3pt}{\color{black!35}\vrule width 0.3pt}\hspace{3pt}}c@{}}
\toprule
\textbf{HunyuanVideo} & \textbf{MiniMax-H3} \\
\midrule
\begin{tabular}[t]{@{}l!{\color{black!35}\vrule width 0.3pt}c!{\color{black!35}\vrule width 0.3pt}*{5}{c}@{}}
Method & \makecell{Bit (\%)\\0 / 4 / 8 / 16} & \makecell{VBench\\Avg$\uparrow$} & \makecell{PSNR\\$\uparrow$} & \makecell{SSIM\\$\uparrow$} & \makecell{LPIPS\\$\downarrow$} & \makecell{Speedup\\$\uparrow$} \\
\midrule
\rowcolor{white}[0.9pt][0.9pt]
Dense & 0/0/0/100 & 77.06 & -- & -- & -- & $1.00\times$ \\
\midrule
\rowcolor{white}[0.9pt][0.9pt]
XAttn & 85.6/0/0/14.4 & 73.90 & 22.90 & 0.679 & 0.251 & $1.74\times$ \\
\rowcolor{white}[0.9pt][0.9pt]
SVG2 & 83.6/0/0/16.4 & 65.70 & 24.34 & 0.711 & 0.382 & $2.06\times$ \\
\rowcolor{white}[0.9pt][0.9pt]
Sol-Attn & 85.8/0/0/14.2 & 76.81 & 25.21 & 0.805 & 0.178 & $2.16\times$ \\
\rowcolor{white}[0.9pt][0.9pt]
SVG-EAR & 85.2/0/0/14.8 & 66.49 & 25.03 & 0.739 & 0.348 & $2.06\times$ \\
\rowcolor{white}[0.9pt][0.9pt]
SVOO & 85.3/0/0/14.7 & 71.37 & 26.72 & 0.816 & 0.174 & $2.15\times$ \\
\rowcolor{oursrow}[0.9pt][0.9pt]
\textbf{Ours} & 85/0/0/15 & 74.55 & 26.19 & 0.822 & 0.151 & $2.28\times$ \\
\rowcolor{oursrow}[0.9pt][0.9pt]
\textbf{Ours} & 85/0/7.5/7.5 & 74.70 & 26.23 & 0.824 & 0.149 & $2.38\times$ \\
\midrule
\rowcolor{white}[0.9pt][0.9pt]
SpargeAttn & 85/0/15/0 & 74.36 & 26.02 & 0.824 & 0.159 & $2.50\times$ \\
\rowcolor{oursrow}[0.9pt][0.9pt]
\textbf{Ours} & 85/0/15/0 & 74.83 & 25.20 & 0.797 & 0.155 & $2.50\times$ \\
\rowcolor{oursrow}[0.9pt][0.9pt]
\textbf{Ours} & 80.7/14.4/4.8/0 & 75.96 & 26.37 & 0.814 & 0.157 & $2.51\times$ \\
\midrule
\rowcolor{white}[0.9pt][0.9pt]
SageAttn3 & 0/100/0/0 & 75.42 & 32.30 & 0.919 & 0.068 & $1.88\times$ \\
\rowcolor{oursrow}[0.9pt][0.9pt]
\textbf{Ours} & 13.4/64.9/21.6/0 & 75.75 & 34.77 & 0.945 & 0.043 & $1.90\times$ \\
\end{tabular}
&
\begin{tabular}[t]{@{}l!{\color{black!35}\vrule width 0.3pt}c!{\color{black!35}\vrule width 0.3pt}*{5}{c}@{}}
Method & \makecell{Bit (\%)\\0 / 4 / 8 / 16} & \makecell{VBench\\Avg$\uparrow$} & \makecell{PSNR\\$\uparrow$} & \makecell{SSIM\\$\uparrow$} & \makecell{LPIPS\\$\downarrow$} & \makecell{Speedup\\$\uparrow$} \\
\midrule
\rowcolor{white}[0.9pt][0.9pt]
Dense & 0/0/0/100 & 78.37 & -- & -- & -- & $1.00\times$ \\
\midrule
\rowcolor{white}[0.9pt][0.9pt]
XAttn & 79.9/0/0/20.1 & 78.03 & 19.64 & 0.622 & 0.267 & $1.76\times$ \\
\rowcolor{white}[0.9pt][0.9pt]
SVG2 & 79.9/0/0/20.1 & 77.02 & 19.87 & 0.626 & 0.293 & $1.99\times$ \\
\rowcolor{white}[0.9pt][0.9pt]
Sol-Attn & 80/0/0/20 & 78.25 & 20.95 & 0.691 & 0.198 & $2.04\times$ \\
\rowcolor{white}[0.9pt][0.9pt]
SVG-EAR & 79.9/0/0/20.1 & 76.97 & 20.40 & 0.644 & 0.275 & $1.97\times$ \\
\rowcolor{white}[0.9pt][0.9pt]
SVOO & 79.9/0/0/20.1 & 78.07 & 20.56 & 0.662 & 0.238 & $2.04\times$ \\
\rowcolor{oursrow}[0.9pt][0.9pt]
\textbf{Ours} & 80/0/0/20 & 77.90 & 21.01 & 0.692 & 0.199 & $2.23\times$ \\
\rowcolor{oursrow}[0.9pt][0.9pt]
\textbf{Ours} & 80/0/10/10 & 77.84 & 20.95 & 0.690 & 0.200 & $2.33\times$ \\
\midrule
\rowcolor{white}[0.9pt][0.9pt]
SpargeAttn & 80/0/20/0 & 77.47 & 20.23 & 0.662 & 0.232 & $2.59\times$ \\
\rowcolor{oursrow}[0.9pt][0.9pt]
\textbf{Ours} & 80/0/20/0 & 78.15 & 20.29 & 0.667 & 0.216 & $2.47\times$ \\
\rowcolor{oursrow}[0.9pt][0.9pt]
\textbf{Ours} & 74.3/19.3/6.4/0 & 77.74 & 21.02 & 0.690 & 0.205 & $2.60\times$ \\
\midrule
\rowcolor{white}[0.9pt][0.9pt]
SageAttn3 & 0/100/0/0 & 77.40 & 23.75 & 0.779 & 0.149 & $1.94\times$ \\
\rowcolor{oursrow}[0.9pt][0.9pt]
\textbf{Ours} & 13.6/64.8/21.6/0 & 77.43 & 25.07 & 0.817 & 0.122 & $1.95\times$ \\
\end{tabular}
\\
\bottomrule
\end{tabular}%
}
\end{table}

%% file: tables/fewstep-results.tex
\begin{table}[!t]
\centering
\captionsetup{font=small,skip=3pt}
\caption{Video generation quality and DiT speedup for MiniMax-H3 with FL2VA Turbo LoRAs ($1344\times768$, 362 frames). Bit allocations follow the 0 / 4 / 8 / 16 order; 0-bit denotes sparse and 16-bit denotes dense.}
\label{tab:fewstep-results}
\fontsize{7.5}{9}\selectfont
\setlength{\tabcolsep}{0.9pt}
\renewcommand{\arraystretch}{1.00}
\setlength{\heavyrulewidth}{0.6pt}
\setlength{\lightrulewidth}{0.3pt}
\arrayrulecolor{black!65}
\definecolor{oursrow}{RGB}{224,237,250}
\resizebox{\linewidth}{!}{%
\begin{tabular}{@{}c@{\hspace{3pt}{\color{black!35}\vrule width 0.3pt}\hspace{3pt}}c@{}}
\toprule
\textbf{4-Step Diffusion} & \textbf{8-Step Diffusion} \\
\midrule
\begin{tabular}[t]{@{}l!{\color{black!35}\vrule width 0.3pt}c!{\color{black!35}\vrule width 0.3pt}*{5}{c}@{}}
Method & \makecell{Bit (\%)\\0 / 4 / 8 / 16} & \makecell{VBench\\Avg$\uparrow$} & \makecell{PSNR\\$\uparrow$} & \makecell{SSIM\\$\uparrow$} & \makecell{LPIPS\\$\downarrow$} & \makecell{Speedup\\$\uparrow$} \\
\midrule
\rowcolor{white}[0.9pt][0.9pt]
Dense & 0/0/0/100 & 74.14 & -- & -- & -- & $1.00\times$ \\
\midrule
\rowcolor{white}[0.9pt][0.9pt]
Sol-Attn & 65/0/0/35 & 74.19 & 19.54 & 0.664 & 0.191 & $1.46\times$ \\
\rowcolor{white}[0.9pt][0.9pt]
SVG2 & 64/0/0/36 & 73.80 & 19.02 & 0.603 & 0.231 & $1.33\times$ \\
\rowcolor{white}[0.9pt][0.9pt]
SVG2 & 85.8/0/0/14.2 & 73.13 & 17.70 & 0.511 & 0.306 & $1.59\times$ \\
\rowcolor{white}[0.9pt][0.9pt]
SVG-EAR & 66.4/0/0/33.6 & 73.81 & 19.39 & 0.619 & 0.224 & $1.32\times$ \\
\rowcolor{white}[0.9pt][0.9pt]
SVG-EAR & 86.1/0/0/13.9 & 73.73 & 18.28 & 0.543 & 0.283 & $1.59\times$ \\
\rowcolor{white}[0.9pt][0.9pt]
SVOO & 65/0/0/35 & 73.80 & 19.59 & 0.637 & 0.217 & $1.44\times$ \\
\rowcolor{white}[0.9pt][0.9pt]
SVOO & 76.3/0/0/23.7 & 73.74 & 18.41 & 0.561 & 0.267 & $1.60\times$ \\
\rowcolor{oursrow}[0.9pt][0.9pt]
\textbf{Ours} & 65/0/0/35 & 74.11 & 20.19 & 0.691 & 0.175 & $1.63\times$ \\
\rowcolor{oursrow}[0.9pt][0.9pt]
\textbf{Ours} & 63.4/0/18.3/18.3 & 73.89 & 19.84 & 0.671 & 0.188 & $1.66\times$ \\
\rowcolor{oursrow}[0.9pt][0.9pt]
\textbf{Ours} & 53.2/0/37.5/9.4 & 73.85 & 20.60 & 0.702 & 0.173 & $1.64\times$ \\
\rowcolor{oursrow}[0.9pt][0.9pt]
\textbf{Ours} & 29.6/35.2/35.2/0 & 74.05 & 22.38 & 0.779 & 0.128 & $1.65\times$ \\
\midrule
\rowcolor{white}[0.9pt][0.9pt]
SpargeAttn & 65/0/35/0 & 73.74 & 19.26 & 0.641 & 0.206 & $1.86\times$ \\
\rowcolor{oursrow}[0.9pt][0.9pt]
\textbf{Ours} & 65/0/35/0 & 73.65 & 19.21 & 0.616 & 0.239 & $1.85\times$ \\
\rowcolor{oursrow}[0.9pt][0.9pt]
\textbf{Ours} & 54.5/34.1/11.4/0 & 74.00 & 20.28 & 0.687 & 0.179 & $1.84\times$ \\
\midrule
\rowcolor{white}[0.9pt][0.9pt]
SageAttn3 & 0/100/0/0 & 74.00 & 21.06 & 0.721 & 0.162 & $1.65\times$ \\
\rowcolor{oursrow}[0.9pt][0.9pt]
\textbf{Ours} & 13.1/65.2/21.7/0 & 73.83 & 22.63 & 0.786 & 0.126 & $1.64\times$ \\
\end{tabular}
&
\begin{tabular}[t]{@{}l!{\color{black!35}\vrule width 0.3pt}c!{\color{black!35}\vrule width 0.3pt}*{5}{c}@{}}
Method & \makecell{Bit (\%)\\0 / 4 / 8 / 16} & \makecell{VBench\\Avg$\uparrow$} & \makecell{PSNR\\$\uparrow$} & \makecell{SSIM\\$\uparrow$} & \makecell{LPIPS\\$\downarrow$} & \makecell{Speedup\\$\uparrow$} \\
\midrule
\rowcolor{white}[0.9pt][0.9pt]
Dense & 0/0/0/100 & 79.87 & -- & -- & -- & $1.00\times$ \\
\midrule
\rowcolor{white}[0.9pt][0.9pt]
Sol-Attn & 79.9/0/0/20.1 & 79.91 & 18.19 & 0.575 & 0.262 & $1.80\times$ \\
\rowcolor{white}[0.9pt][0.9pt]
SVG2 & 79.6/0/0/20.4 & 79.34 & 17.80 & 0.531 & 0.333 & $1.66\times$ \\
\rowcolor{white}[0.9pt][0.9pt]
SVG2 & 86.7/0/0/13.3 & 79.04 & 17.36 & 0.509 & 0.371 & $1.81\times$ \\
\rowcolor{white}[0.9pt][0.9pt]
SVG-EAR & 79.6/0/0/20.4 & 79.51 & 18.39 & 0.557 & 0.309 & $1.68\times$ \\
\rowcolor{white}[0.9pt][0.9pt]
SVG-EAR & 88.4/0/0/11.6 & 78.81 & 17.87 & 0.530 & 0.350 & $1.82\times$ \\
\rowcolor{white}[0.9pt][0.9pt]
SVOO & 80.1/0/0/19.9 & 79.24 & 17.99 & 0.532 & 0.338 & $1.80\times$ \\
\rowcolor{oursrow}[0.9pt][0.9pt]
\textbf{Ours} & 80/0/0/20 & 79.88 & 18.75 & 0.596 & 0.247 & $1.84\times$ \\
\rowcolor{oursrow}[0.9pt][0.9pt]
\textbf{Ours} & 74.3/0/12.8/12.8 & 79.85 & 18.85 & 0.601 & 0.243 & $1.87\times$ \\
\rowcolor{oursrow}[0.9pt][0.9pt]
\textbf{Ours} & 57.9/0/42.1/0 & 79.69 & 19.15 & 0.615 & 0.245 & $1.87\times$ \\
\rowcolor{oursrow}[0.9pt][0.9pt]
\textbf{Ours} & 40.5/59.5/0/0 & 79.73 & 19.75 & 0.633 & 0.240 & $1.87\times$ \\
\rowcolor{oursrow}[0.9pt][0.9pt]
\textbf{Ours} & 50.8/24.6/24.6/0 & 79.82 & 20.76 & 0.690 & 0.188 & $1.87\times$ \\
\midrule
\rowcolor{white}[0.9pt][0.9pt]
SpargeAttn & 80/0/20/0 & 79.70 & 18.15 & 0.566 & 0.277 & $2.12\times$ \\
\rowcolor{oursrow}[0.9pt][0.9pt]
\textbf{Ours} & 80/0/20/0 & 79.81 & 18.46 & 0.582 & 0.262 & $2.12\times$ \\
\rowcolor{oursrow}[0.9pt][0.9pt]
\textbf{Ours} & 74/13/13/0 & 79.89 & 18.83 & 0.599 & 0.247 & $2.12\times$ \\
\midrule
\rowcolor{white}[0.9pt][0.9pt]
SageAttn3 & 0/100/0/0 & 79.72 & 20.99 & 0.693 & 0.195 & $1.73\times$ \\
\rowcolor{oursrow}[0.9pt][0.9pt]
\textbf{Ours} & 13/65.3/21.8/0 & 79.80 & 22.57 & 0.758 & 0.152 & $1.73\times$ \\
\end{tabular}
\\
\bottomrule
\end{tabular}%
}
\end{table}

%% file: tables/maxpool-ablation.tex
\centering
\captionsetup{font=small,skip=4pt}
\captionof{table}{MaxPool ablation on MiniMax-H3 at 80\% sparsity. The 8/16-bit configuration assigns equal budgets to the two precisions.}
\label{tab:maxpool-ablation}
\fontsize{8}{9.5}\selectfont
\setlength{\tabcolsep}{2.5pt}
\renewcommand{\arraystretch}{1.08}
\begin{tabular}{@{}cccccc@{}}
\toprule
Bit & \makecell{Max\\Pool} & \makecell{VBench\\Avg. (\%)$\uparrow$} & \makecell{PSNR\\$\uparrow$} & \makecell{SSIM\\$\uparrow$} & \makecell{LPIPS\\$\downarrow$} \\
\midrule
\multirow{2}{*}{16} & w/o & 78.20 & 20.89 & 0.687 & 0.202 \\
 & w/ & 77.90 & 21.01 & 0.692 & 0.199 \\
\midrule
\multirow{2}{*}{8/16} & w/o & 77.68 & 20.82 & 0.684 & 0.204 \\
 & w/ & 77.84 & 20.95 & 0.690 & 0.200 \\
\midrule
\multirow{2}{*}{8} & w/o & 78.11 & 20.15 & 0.660 & 0.221 \\
 & w/ & 78.15 & 20.29 & 0.667 & 0.216 \\
\bottomrule
\end{tabular}

%% file: sections/5-conclusion.tex
\section{Conclusion}

We presented DraftAttention2, a training-free attention acceleration framework that uses a low-resolution draft to jointly allocate sparsity and numerical precision in video diffusion.
The draft map provides a unified importance ranking for both block selection and precision assignment, enabling adaptive allocation under configurable computation budgets.
Our theoretical analysis decomposes attention-output error into sparsification and attention-weighted quantization terms, showing that recovering selected sparse interactions with low-bit computation can substantially tighten the output-error bound and yield a tighter bound than the original DraftAttention under matched allocations.
Fused operand preparation and shared accumulation across precision phases reduce the overhead of executing these fine-grained assignments, translating the proposed allocation into practical acceleration.
Experimental results on HunyuanVideo and MiniMax-H3 demonstrate a strong quality--efficiency trade-off over existing efficient attention methods.
Notably, DraftAttention2 shows particularly strong advantages on few-step video diffusion, where mixed-precision recovery substantially improves generation quality and dense-output fidelity while retaining significant acceleration.

%% file: sections/6-appendix.tex
\clearpage
\appendix
\section*{Appendix}

\section{Additional Experimental Details}
\label{sec:experimental-details}

\paragraph{Generation and attention settings.}
We use the released HunyuanVideo-13B and MiniMax-H3 checkpoints.
For HunyuanVideo, we use 50 denoising steps, a guidance scale of 6, and a timestep shift of 7.
MiniMax-H3 follows its released sampling schedule with a video timestep shift of 12.
The few-step models apply FL2VA Turbo 4-step and 8-step v1.0 768p LoRAs to MiniMax-H3, with a video timestep shift of 6 for both variants.
The standard models retain dense attention for the first 10 denoising steps.
For the few-step models, we retain dense attention for the first step in the 4-step variant and the first two steps in the 8-step variant, as well as the first two layers of every step.
We measure latency over the complete DiT denoising trajectory, including attention preparation and retained dense attention, and report speedup relative to the corresponding dense baseline.

\paragraph{Quality evaluation.}
We evaluate generation quality using the VBench implementation and the prompt sets associated with the reported dimensions.
The eight dimensions are subject consistency (SC), background consistency (BC), temporal flickering (TF), motion smoothness (MS), aesthetic quality (AQ), imaging quality (IQ), dynamic degree (DD), and overall consistency (OC).
We report each score as a percentage and use their arithmetic mean as Avg.
To measure fidelity to dense generation, each output is paired with its dense-attention reference using the same prompt and random seed.
PSNR and SSIM are computed on RGB frames normalized to $[0,1]$, and LPIPS uses the AlexNet backbone.
We evaluate similarity over all frames, average the frame scores within each video, and then average across videos.

\paragraph{Attention error analysis.}
For attention error analysis, we use a dense-trajectory capture from the 4-step model, with 109,376 tokens, 56 heads, and head dimension 128.
Relative L2 error is $\lVert\widehat O-O^{\mathrm{dense}}\rVert_F/\lVert O^{\mathrm{dense}}\rVert_F$, reported as a percentage.
Latency is the median of three measurements of the complete attention operator.
Each curve connects non-dominated measurements at a fixed precision ratio as sparsity varies.
For comparisons under a latency budget, we select the lowest measured error among configurations whose measured latency does not exceed that budget; the single-precision reference includes the 4-, 8-, and 16-bit paths.

\input{tables/vbench-main-results}
\input{tables/vbench-fewstep-results}

\section{Quantization Analysis and Proofs}
\label{app:v2-quantization}

We analyze one attention head and one video query at a time.
All sums exclude padding and include the retained prefix interactions.
The support is fixed by the router before the mixed-precision computation.
As in DraftAttention~\citep{shen2025draftattention}, we separate the effect of removing interactions from the approximation on the retained support.
Tracking both the output numerator and the softmax denominator also underlies the approximation analysis in Sol-Attn~\citep{li2026solattn}.
Here the retained interactions are computed with quantized operands, and the analysis explicitly accounts for the different weights used by the denominator and the probability--value multiplication.

\subsection{Reconstruction-Error Model}
\label{app:v2-error-model}

We take the stored BF16 tensors $Q,K,V$ as reference inputs.
The reconstruction model isolates additional low-bit quantization: on the 16-bit path, $\widehat s_{uv}=s_{uv}$, $\widehat V_v=V_v$, and $\widehat w_{uv}=w_{uv}$.
Thus $\alpha_{ub}=\nu_{ub}=\omega_{ub}=0$ on this path.
Ordinary floating-point effects, including any reference-path format conversion, are accounted for separately in $\varepsilon_u^{\rm arith}$ below.
For completeness, the exact dense operator on these reference inputs is
\begin{equation}
 P=\operatorname{softmax}_{\rm row}(QK^\top/\sqrt d),
 \qquad O^{\rm dense}=PV.
 \label{eq:v2-dense}
\end{equation}
For a query $u$, write $s_{uv}=Q_u^\top K_v/\sqrt d$ and
$P^\Omega_{uv}=e^{s_{uv}}/\sum_{r\in\Omega_u}e^{s_{ur}}$ on its retained support.
For a retained key block $\mathcal B_b$, write
$m_{ub}=\sum_{v\in\mathcal B_b}P^\Omega_{uv}$.
The precision-level mass in Equation~\eqref{eq:v2-block-mass} is
$m_u^{(p)}=\sum_{b:p_{ub}=p}m_{ub}$.
Let $\widehat s_{uv}$ and $\widehat V_v$ be the reconstructed logit and value.
At a common row reference $h_u$, set $w_{uv}=e^{\widehat s_{uv}-h_u}$.
The denominator sums $w_{uv}$ before weight quantization; the numerator uses reconstructed weights $\widehat w_{uv}$, including subsequent online rescalings.
For each retained block, assume
\begin{equation}
 \begin{aligned}
 |\widehat s_{uv}-s_{uv}|&\le\alpha_{ub},
 &\|\widehat V_v-V_v\|_2&\le\nu_{ub},\\
 \sum_{v\in\mathcal B_b}|\widehat w_{uv}-w_{uv}|
 &\le\omega_{ub}\sum_{v\in\mathcal B_b}w_{uv}.
 \end{aligned}
 \label{eq:v2-quant-errors}
\end{equation}
These are bounds on reconstructed operands in their original units.
Define $\alpha_u=\max_b\alpha_{ub}$ and the block error cost
\begin{equation}
 c_{ub}=2V_{\max}(e^{\alpha_{ub}}-1)
 +e^{\alpha_{ub}}\bigl[\nu_{ub}+(V_{\max}+\nu_{ub})\omega_{ub}\bigr].
 \label{eq:v2-quant-cost}
\end{equation}
For the format-level statement in the main text, assume common bounds
$\alpha_{ub}\le\alpha_p$, $\nu_{ub}\le\nu_p$, and $\omega_{ub}\le\omega_p$
whenever block $b$ uses low-bit precision $p\in\{4,8\}$.
These bounds cover all admissible assignments under consideration.
With $\bar\alpha=\max(\alpha_4,\alpha_8)$, define
\begin{equation}
 \varepsilon_p=e^{\bar\alpha}
 \left\{2V_{\max}(e^{\alpha_p}-1)
 +e^{\alpha_p}\bigl[\nu_p+(V_{\max}+\nu_p)\omega_p\bigr]\right\}.
 \label{eq:v2-format-envelope}
\end{equation}
For notation in the proofs, set $\varepsilon_{16}=0$: reference-precision blocks have no additional low-bit quantization cost.
For small errors, the leading terms are
$2V_{\max}\alpha_p+\nu_p+V_{\max}\omega_p$.
The ordering of these envelopes is an assumption of Theorem~\ref{prop:v2-allocation}, not a consequence of bit width alone.

\subsection{Reconstructed Operands and Logit Error}
\label{app:v2-operands}

For a fixed query $u$, the retained key supports $\{\mathcal B_b\}$ are disjoint and partition $\Omega_u$.
All quantities below are expressed in the original numerical units after restoring their quantization scales.
For a block assigned precision $p$, let
\[
 \widehat Q_u^{(p)}=Q_u+\Delta Q_u^{(p)},\qquad
 \widehat K_v^{(p)}=K_v+\Delta K_v^{(p)},\qquad
 \widehat V_v=V_v+\Delta V_v.
\]
The reconstructed logit is
$\widehat s_{uv}=(\widehat Q_u^{(p)})^\top\widehat K_v^{(p)}/\sqrt d$.
Expanding the inner product gives
\begin{equation}
 |\widehat s_{uv}-s_{uv}|
 \le
 \frac{
 \|Q_u\|_2\|\Delta K_v^{(p)}\|_2+
 \|K_v\|_2\|\Delta Q_u^{(p)}\|_2+
 \|\Delta Q_u^{(p)}\|_2\|\Delta K_v^{(p)}\|_2
 }{\sqrt d}.
 \label{eq:v2-logit-quant}
\end{equation}
Taking the maximum over $v\in\mathcal B_b$ provides a valid $\alpha_{ub}$ in Equation~\eqref{eq:v2-quant-errors}.
The value error is bounded by
$\nu_{ub}=\max_{v\in\mathcal B_b}\|\Delta V_v\|_2$.
For the low-bit paths, these residuals include errors from operand conversion and the represented quantization scales.

For a uniform nearest-integer quantizer with recovery step $\Delta$ and no clipping, each reconstructed scalar has error at most $\Delta/2$.
A $d$-dimensional vector with a common step therefore has error at most $\sqrt d\,\Delta/2$.
This observation applies to the INT8 operand path under its range assumptions.
Floating-point formats such as E2M1 and E4M3 have nonuniform spacing; for these paths, Equation~\eqref{eq:v2-logit-quant} uses the actual reconstruction residuals.
The theoretical costs need not be computed by the router.

\subsection{Weight Quantization under Shared Normalization}
\label{app:v2-shared-normalization}

Choose a common final row reference $h_u$, for example the maximum reconstructed logit on $\Omega_u$, and write
\[
 w_{uv}=e^{\widehat s_{uv}-h_u},\qquad
 Z_u^{\rm q}=\sum_{v\in\Omega_u}w_{uv}.
\]
The denominator uses these weights before their conversion to the operand format of the probability--value multiplication.
Let $\widehat w_{uv}$ denote the reconstructed weight used in that multiplication, expressed at the same final reference.
Then, in exact arithmetic on these reconstructed operands,
\begin{equation}
 \widehat O_u=
 \frac{\sum_{v\in\Omega_u}\widehat w_{uv}\widehat V_v}
      {Z_u^{\rm q}}.
 \label{eq:v2-reconstructed-output}
\end{equation}
In general, $\sum_v\widehat w_{uv}$ differs from $Z_u^{\rm q}$.
Equation~\eqref{eq:v2-reconstructed-output} retains this distinction.

To relate this representation to online execution, suppose block $b$ is processed at running row maximum $h_u^{(b)}$.
Its unquantized weights are $e^{\widehat s_{uv}-h_u^{(b)}}$.
Every later increase of the row maximum rescales the previous numerator and denominator by the same factor.
The product of these factors is $e^{h_u^{(b)}-h_u}$.
Thus the block's final denominator contribution is $\sum_{v\in\mathcal B_b}w_{uv}$, and its reconstructed numerator weights are the locally quantized weights multiplied by $e^{h_u^{(b)}-h_u}$.
This defines $\widehat w_{uv}$ for the actual traversal order.
In particular, the block-relative error
\[
 \frac{\sum_{v\in\mathcal B_b}|\widehat w_{uv}-w_{uv}|}
      {\sum_{v\in\mathcal B_b}w_{uv}}
\]
is unchanged by these common rescalings.

The phase-boundary transformations in Equation~\eqref{eq:precision-state-conversion} are changes of numerical units.
In exact arithmetic, the entry and exit scalings are inverse transformations, and the normalized output is preserved after restoring the value scale.
Consequently, all precision phases contribute to the single numerator and denominator in Equation~\eqref{eq:v2-reconstructed-output}.
Rounding in exponentiation, scale conversion, matrix accumulation, and final output conversion is accounted for separately below.

\subsection{Proof of the Attention-Weighted Error Bound}
\label{app:v2-error-proof}

\begin{proof}[Proof of Lemma~\ref{thm:v2-quantization}]
Fix a query $u$ with nonempty retained support and define its exact retained-support output by
\[
 O_u^\Omega=\sum_{v\in\Omega_u}P^\Omega_{uv}V_v.
\]
We first bound the support-removal error.
Let
$D_u=\sum_{v\notin\Omega_u}P_{uv}V_v$.
Since the retained dense probability mass is $1-\delta_u$,
\[
 O_u^{\rm dense}=(1-\delta_u)O_u^\Omega+D_u.
\]
The value bound implies
$\|D_u\|_2\le\delta_u V_{\max}$ and $\|O_u^\Omega\|_2\le V_{\max}$, hence
\begin{equation}
 \|O_u^{\rm dense}-O_u^\Omega\|_2\le2V_{\max}\delta_u.
 \label{eq:v2-support-proof}
\end{equation}

Next, let $z_{uv}=e^{s_{uv}-h_u}$ and $Z_u=\sum_{v\in\Omega_u}z_{uv}$.
Thus $P^\Omega_{uv}=z_{uv}/Z_u$ and
$\sum_{v\in\Omega_u}z_{uv}(V_v-O_u^\Omega)=0$.
Subtracting $O_u^\Omega$ from Equation~\eqref{eq:v2-reconstructed-output} yields the exact identity
\begin{align}
 Z_u^{\rm q}(\widehat O_u-O_u^\Omega)
 &=
 \sum_{v\in\Omega_u}(w_{uv}-z_{uv})(V_v-O_u^\Omega)
 \nonumber\\
 &\quad+
 \sum_{v\in\Omega_u}w_{uv}(\widehat V_v-V_v)
 +
 \sum_{v\in\Omega_u}(\widehat w_{uv}-w_{uv})\widehat V_v.
 \label{eq:v2-numerator-decomposition}
\end{align}
The three terms respectively account for logit, value, and exponential-weight quantization, including the change in normalization caused by the perturbed logits.

For $v\in\mathcal B_b$, the logit bound gives
\[
 |w_{uv}-z_{uv}|\le z_{uv}(e^{\alpha_{ub}}-1),
 \qquad
 w_{uv}\le e^{\alpha_{ub}}z_{uv}.
\]
Also,
$\|V_v-O_u^\Omega\|_2\le2V_{\max}$ and
$\|\widehat V_v\|_2\le V_{\max}+\nu_{ub}$.
Applying these inequalities and Equation~\eqref{eq:v2-quant-errors} to the contribution of each block in Equation~\eqref{eq:v2-numerator-decomposition} gives
\[
 Z_u^{\rm q}\|\widehat O_u-O_u^\Omega\|_2
 \le
 \sum_b
 \left[
 2V_{\max}(e^{\alpha_{ub}}-1)
 +e^{\alpha_{ub}}\{\nu_{ub}+(V_{\max}+\nu_{ub})\omega_{ub}\}
 \right]
 \sum_{v\in\mathcal B_b}z_{uv}.
\]
Since $\alpha_u=\max_b\alpha_{ub}$,
\[
 Z_u^{\rm q}\ge e^{-\alpha_u}Z_u,
 \qquad
 \sum_{v\in\mathcal B_b}z_{uv}=Z_um_{ub}.
\]
Dividing by $Z_u^{\rm q}$ establishes
\begin{equation}
 \|\widehat O_u-O_u^\Omega\|_2
 \le e^{\alpha_u}\sum_bm_{ub}c_{ub}.
 \label{eq:v2-quant-proof}
\end{equation}
Combining Equations~\eqref{eq:v2-support-proof} and \eqref{eq:v2-quant-proof} gives the block-dependent bound.
Since $e^{\alpha_u}c_{ub}\le\varepsilon_{p_{ub}}$ under the common format bounds, Equation~\eqref{eq:v2-total-error} follows.
\end{proof}

\paragraph{Arithmetic roundoff and empty rows.}
Let $O_u^{\rm impl}$ denote the implemented output.
If
$\|O_u^{\rm impl}-\widehat O_u\|_2\le\varepsilon_u^{\rm arith}$,
then
\[
 \|O_u^{\rm dense}-O_u^{\rm impl}\|_2
 \le2V_{\max}\delta_u+
 e^{\alpha_u}\sum_bm_{ub}c_{ub}+
 \varepsilon_u^{\rm arith}.
\]
This last term includes ordinary floating-point rounding and any reference-path format conversion, including effects on subsequent shared normalization.
Omitting a 16-bit quantization term therefore does not assume exact hardware arithmetic.
If a query has no valid retained key, the implementation returns zero, and
$\|O_u^{\rm dense}-O_u^{\rm impl}\|_2\le V_{\max}$ directly.

\subsection{Proof of the v1-to-v2 Error-Bound Comparison}
\label{app:v2-recovery}

\begin{proof}[Proof of Proposition~\ref{prop:v2-recovery}]
Fix a query and let $\Omega_{16}$ be the support retained by v1 and preserved at 16 bits by v2.
Let $L_8,L_4$ be disjoint subsets of its complement, with $\tau_p=\sum_{v\in L_p}P_{uv}$.
The v2 support is $\Omega'=\Omega_{16}\cup L_8\cup L_4$, with dense attention mass $1-\delta_{\rm v2}$.
Hence $\Omega_{16}$ has mass $1-\delta_{\rm v2}-\tau>0$.
The v1 allocation skips mass $\delta_{\rm v1}=\delta_{\rm v2}+\tau$ and has no additional low-bit quantization cost, so Lemma~\ref{thm:v2-quantization} gives
$B_{\rm v1}=2V_{\max}(\delta_{\rm v2}+\tau)$.
On $\Omega'$, the exact renormalized masses are
\[
 m_u^{(16)}=\frac{1-\delta_{\rm v2}-\tau}{1-\delta_{\rm v2}},\qquad
 m_u^{(8)}=\frac{\tau_8}{1-\delta_{\rm v2}},\qquad
 m_u^{(4)}=\frac{\tau_4}{1-\delta_{\rm v2}}.
\]
Applying the lemma to v2 gives $B_{\rm v2}$ in Equation~\eqref{eq:v2-recovery-bounds}.
Subtracting the two bounds yields
\begin{equation}
 B_{\rm v2}-B_{\rm v1}
 =-2V_{\max}\tau+
 \frac{\tau_8\varepsilon_8+\tau_4\varepsilon_4}{1-\delta_{\rm v2}}.
 \label{eq:v2-recovery-difference}
\end{equation}
Since $1-\delta_{\rm v2}>0$, Equation~\eqref{eq:v2-recovery-condition} is precisely the condition for this difference to be negative.
\end{proof}

The comparison concerns output-error bounds for reconstructed operators with matched 16-bit supports.
It isolates v2's low-bit recovery from changes in draft scoring or support selection.
If $\rho_{\rm v1}$ and $\rho_{\rm v2}$ denote the skipped fractions of all video-region pairs, and $\rho_4,\rho_8$ the fractions assigned to the low-bit paths by v2, this construction gives $\rho_{\rm v1}=\rho_{\rm v2}+\rho_4+\rho_8$.
The quantities $\delta_{\rm v2},\tau_8,\tau_4$ are dense attention masses, not block fractions; equal sparsity ratios alone do not establish the comparison.
Arithmetic roundoff can be added separately as in Appendix~\ref{app:v2-error-proof}; measured latency and relative output error are evaluated by the operator experiments.

\subsection{Proof of Importance-Guided Precision Allocation}
\label{app:v2-precision-allocation}

Fix the retained video-region pairs and the precision of all always-retained interactions.
For the allocation objective, set $P^\Omega_{uv}=0$ on rows with empty retained support; their zero outputs contribute a fixed support error.
For a retained region pair $(a,b)$, its contribution to query-averaged attention mass is
\[
 W_{ab}=\frac1{N_{\rm vid}}
 \sum_{u\in\mathcal R_a}\sum_{v\in\mathcal R_b}P^\Omega_{uv},
 \qquad N_{\rm vid}=\sum_a|\mathcal R_a|.
\]
The inner sum includes all physical blocks inherited from region $b$.
With the support fixed, these weights are independent of precision.
Averaging Lemma~\ref{thm:v2-quantization} over video queries gives the assignment-dependent objective
\begin{equation}
 \sum_{(a,b)\ {\rm retained}}W_{ab}\varepsilon_{p_{ab}}.
 \label{eq:v2-allocation-objective}
\end{equation}
The support error and always-retained interactions with fixed precision contribute constants.

\begin{proof}[Proof of Theorem~\ref{prop:v2-allocation}]
Consider two retained pairs with $W_1\ge W_2$ and two format costs
$\varepsilon_{\rm low}\le\varepsilon_{\rm high}$.
The excess cost of assigning the larger error cost to the larger weight is
\[
 (W_1\varepsilon_{\rm high}+W_2\varepsilon_{\rm low})
 -(W_1\varepsilon_{\rm low}+W_2\varepsilon_{\rm high})
 =(W_1-W_2)(\varepsilon_{\rm high}-\varepsilon_{\rm low})\ge0.
\]
Exchanging each inverted assignment preserves the quotas and cannot increase the objective.
Repeated exchanges yield the stated ordering.
\end{proof}

The theorem concerns the error bound at fixed support and quotas.
The implemented router approximates its mass ordering with $A_{ab}$; it does not evaluate $W_{ab}$ or the error envelopes online.
The pooling mixture controls this proxy and does not enter the quantization-error decomposition.

\section{Implementation Details}
\label{app:v2-implementation}

Block-wise precision allocation requires multiple operand formats while retaining a common normalization across all selected blocks.
We fuse operand preparation to reduce data movement and organize computation into precision-specific phases that share one online softmax state.

\subsection{Routing and Precision Formats}
\label{app:v2-routing-details}

For each spatial region $\mathcal R_a$, valid-token counts $n_a=|\mathcal R_a|$ exclude padding.
The descriptors in Equation~\eqref{eq:v2-draft} are
$\bar Q_a=n_a^{-1}\sum_{u\in\mathcal R_a}Q_u$ and
$Q^{\max}_{a,c}=\max_{u\in\mathcal R_a}Q_{u,c}$, with keys defined analogously.
Both branches normalize over key regions.
The router ranks all $R^2$ video-region pairs within each head, resolving ties by region-pair index.
Precision fractions determine integer quotas $T_p$ satisfying $\sum_pT_p=T$.
Each region pair's precision is assigned to all of its physical attention blocks, so the number of retained blocks and their precisions can vary across query regions.
The labels $4$, $8$, and $16$ denote the NVFP4, INT8, and 16-bit paths.
The INT8 path uses INT8 for $QK^\top$ and FP8 for the probability--value multiplication.
Prefix interactions lie outside the video-region budget: video queries retain all valid prefix keys, and prefix queries use non-sparse attention.
Rows with no valid retained key return zero.

\subsection{Fused Quantization and Reordering}
\label{sec:fused-preparation}

Section~\ref{sec:hardware} describes the fused preparation shown in Figure~\ref{fig:fused-preparation}.
Here we give the spatial packing and operand-layout details.

We partition each $H_v\times W_v$ frame into $R_f$ connected regions and pack each region into a contiguous block of capacity $C$, where
\begin{equation}
 R_f=\left\lceil\frac{H_vW_v}{C}\right\rceil,
 \qquad 0\le R_fC-H_vW_v<C.
 \label{eq:ragged-padding}
\end{equation}
This ragged 2D layout uses the minimum number of blocks without requiring a fixed patch shape to divide the frame dimensions.
To preserve spatial locality, we prioritize the smallest total perimeter among candidate partitions.
For a rectangular token grid,
\begin{equation}
 \sum_{a=0}^{R_f-1}|\partial\mathcal R_a|
 =2(H_v+W_v)+2E_{\mathrm{cut}},
 \label{eq:region-perimeter}
\end{equation}
where $|\partial\mathcal R_a|$ counts unit boundary edges and $E_{\mathrm{cut}}$ counts horizontal or vertical neighbor pairs assigned to different regions.
Reducing total perimeter therefore preserves more spatial neighbors within the same pooled region.
Valid-token counts exclude padding from pooling and attention; the inverse index map restores the output order.

The $QK^\top$ accumulator layout can differ from the operand layout required by the subsequent probability--value multiplication~\citep{shah2024flashattention3}.
We align these layouts so that $\widetilde P$ can feed the subsequent matrix multiplication through thread-local conversion and packing.
For an intra-block permutation matrix $\Pi$, the matching column and row permutations preserve the product:
\begin{equation}
 (\widetilde P\Pi)(\Pi^\top V)=\widetilde PV.
 \label{eq:operand-permutation}
\end{equation}
We fuse the required intra-block $K/V$ permutations with quantization, and incorporate $V$ transposition into operand packing and loading.

\subsection{Quantization Formats and Scales}
\label{app:v2-quantization-formats}

Table~\ref{tab:v2-arithmetic} distinguishes matrix operand formats from accumulation formats.
The row maxima, exponential sums, and output state shared across blocks and phases are stored in FP32.
For the INT8 and FP16 paths, each physical key block's FP16 probability--value accumulator is converted to FP32 before being added to the persistent output state.

\begin{table}[htbp]
 \centering
 \small
 \caption{Operand and accumulator formats in the mixed-precision kernel.}
 \label{tab:v2-arithmetic}
 \setlength{\tabcolsep}{9pt}
 \begin{tabular}{lcccc}
  \toprule
  Path & $Q,K$ & $QK^\top$ accumulator & $\widetilde P,V$ & $\widetilde PV$ accumulator\\
  \midrule
  NVFP4 & E2M1 & FP32 & E2M1 & FP32\\
  INT8 & INT8 & INT32 & E4M3 & FP16\\
  FP16 & FP16 & FP32 & FP16 & FP16\\
  \bottomrule
 \end{tabular}
\end{table}

\paragraph{NVFP4 operands.}
We use E2M1 values with E4M3 local scales for groups of 16 elements along each matrix multiplication's reduction dimension.
For $Q$ and $K$, groups run along feature channels; for $V$, they run along key tokens within each output channel.
Given a positive per-tensor scale $g_X$ for $X\in\{Q,K,V\}$ and a quantization group $G$, the local scale and reconstructed values are
\begin{equation}
 \sigma_{X,G}=\operatorname{cast}_{\mathrm{E4M3}}
 \!\left(\frac{\max_{x\in G}|x|}{6g_X}\right),
 \qquad
 \widehat x=g_X\sigma_{X,G}\operatorname{cast}_{\mathrm{E2M1}}
 \!\left(\frac{x}{g_X\sigma_{X,G}}\right).
 \label{eq:v2-nvfp4-quantization}
\end{equation}
A group whose encoded scale is zero is encoded as zero.
The matrix instructions apply local scales, and $g_Qg_K/\sqrt d$ restores the scale of the attention logits.

The exponential weights $\widetilde P$ are quantized online.
For a query $u$ and a group $G$ of 16 key positions, we use
\begin{equation}
 \begin{aligned}
 \sigma_{P,u,G}
 &=\operatorname{cast}_{\mathrm{E4M3}}
   \!\left(448\max_{v\in G}\widetilde P_{uv}\right),\\
 \widehat{\widetilde P}_{uv}
 &=\frac{\sigma_{P,u,G}}{2688}
   \operatorname{cast}_{\mathrm{E2M1}}
   \!\left(\frac{2688\widetilde P_{uv}}{\sigma_{P,u,G}}\right).
 \end{aligned}
 \label{eq:v2-nvfp4-probability}
\end{equation}
Here $2688=6\cdot448$ combines the maximum finite values of E2M1 and E4M3; zero-scale groups again produce zeros.
The denominator is updated with the FP32 weights before this conversion.

\paragraph{INT8 and FP16 operands.}
The INT8 path quantizes $Q$ and $K$ with one FP32 scale per query block and physical key block, respectively.
For either operand block $X$, the scale is $\delta_X=\max|X|/127+10^{-7}$ and the encoding rounds $X/\delta_X$ to the nearest integer, with ties away from zero.
The INT32 dot products are converted to FP32 and rescaled by $\delta_Q\delta_K/\sqrt d$.
For the probability--value product, $V$ is encoded in E4M3 with the per-channel scale $\delta_c^V=\max_v|V_{v,c}|/2.25$; all-zero channels are encoded as zero.
Exponentiation uses base two with an additive offset of $8.807$, amplifying the weights by approximately $448$ before E4M3 conversion.
The probability dequantization scale is $\eta=0.0022326917$.
The denominator accumulates the amplified FP32 weights, while the numerator uses E4M3 operands.
The FP16 path uses FP16 $Q,K,V$ and converts exponential weights to FP16 after updating the FP32 denominator.

\subsection{Phased Mixed-Precision Execution}
\label{sec:phased-execution}

The NVFP4, INT8, and FP16 phases extend the same online-softmax state described in Section~\ref{sec:hardware}.
Changes in the running row maximum rescale both the existing numerator and denominator before new block contributions are accumulated.
Phase-boundary transformations account separately for the numerical scales of the probability and value operands.

The NVFP4 phase executes first.
Its matrix instructions incorporate the local scales of $\widetilde P$ and $V$; multiplying the resulting numerator by $g_V/2688$ restores the common accumulation scale.
Its denominator already accumulates unamplified exponential weights.
The INT8 phase enters and leaves its probability and value scales using Equation~\eqref{eq:precision-state-conversion}; the running row maximum is unchanged by this conversion.
The FP16 phase then continues the accumulation in the common scale.
Final normalization divides the output numerator by the shared denominator, and the inverse spatial index map restores token order.

The common shared-memory workspace is sized for the largest phase or output-staging requirement.
Each phase's operand fragments, quantization scales, and loading temporaries remain local to that phase, allowing their storage to be reused by subsequent phases.

\input{sections/appendix-visualizations}

%% file: tables/vbench-main-results.tex
\begingroup
\captionsetup{font=small,skip=3pt}
\fontsize{8.5}{10}\selectfont
\setlength{\tabcolsep}{2.2pt}
\renewcommand{\arraystretch}{1.00}
\setlength{\heavyrulewidth}{0.6pt}
\setlength{\lightrulewidth}{0.3pt}
\arrayrulecolor{black!65}
\definecolor{oursrow}{RGB}{224,237,250}
\setlength{\LTleft}{\fill}
\setlength{\LTright}{\fill}
\setlength{\LTcapwidth}{\textwidth}
\begin{longtable}{@{}l!{\color{black!35}\vrule width 0.3pt}c!{\color{black!35}\vrule width 0.3pt}*{9}{c}@{}}
\caption{Detailed VBench scores (\%) for HunyuanVideo and MiniMax-H3. Configurations correspond to Table~\ref{tab:main-results}.}\label{tab:vbench-main-results}\\
\toprule
\multirow{2}{*}{Method} & \multirow{2}{*}{\makecell{Bit (\%)\\0 / 4 / 8 / 16}} & \multicolumn{9}{c}{VBench$\uparrow$} \\
\cmidrule(lr){3-11}
 & & SC & BC & TF & MS & AQ & IQ & DD & OC & AVG \\
\endfirsthead
\multicolumn{11}{c}{\tablename\ \thetable\ (continued)} \\
\toprule
\multirow{2}{*}{Method} & \multirow{2}{*}{\makecell{Bit (\%)\\0 / 4 / 8 / 16}} & \multicolumn{9}{c}{VBench$\uparrow$} \\
\cmidrule(lr){3-11}
 & & SC & BC & TF & MS & AQ & IQ & DD & OC & AVG \\
\midrule
\endhead
\bottomrule
\endfoot
\midrule
\multicolumn{11}{c}{\textit{\textbf{HunyuanVideo}}} \\*
\midrule
Dense & 0/0/0/100 & 93.87 & 96.60 & 98.97 & 99.11 & 61.54 & 67.87 & 72.22 & 26.28 & 77.06 \\
XAttn & 85.6/0/0/14.4 & 93.76 & 95.57 & 99.05 & 98.76 & 60.10 & 66.38 & 51.39 & 26.16 & 73.90 \\
SVG2 & 83.6/0/0/16.4 & 93.38 & 96.88 & 99.57 & 99.37 & 50.30 & 51.24 & 9.72 & 25.15 & 65.70 \\
Sol-Attn & 85.8/0/0/14.2 & 93.62 & 96.49 & 98.95 & 99.05 & 61.52 & 67.55 & 70.83 & 26.50 & 76.81 \\
SVG-EAR & 85.2/0/0/14.8 & 93.83 & 96.99 & 99.65 & 99.43 & 51.41 & 52.91 & 12.50 & 25.20 & 66.49 \\
SVOO & 85.3/0/0/14.7 & 94.30 & 96.68 & 99.61 & 99.38 & 58.82 & 63.27 & 33.33 & 25.58 & 71.37 \\
\rowcolor{oursrow}[2.2pt][2.2pt]
\textbf{Ours} & 85/0/0/15 & 94.76 & 97.23 & 99.69 & 99.48 & 60.96 & 68.10 & 50.00 & 26.16 & 74.55 \\
\rowcolor{oursrow}[2.2pt][2.2pt]
\textbf{Ours} & 85/0/7.5/7.5 & 94.83 & 97.33 & 99.70 & 99.49 & 60.96 & 67.81 & 51.39 & 26.10 & 74.70 \\
\midrule
SpargeAttn & 85/0/15/0 & 94.86 & 97.33 & 99.71 & 99.50 & 60.93 & 67.89 & 48.61 & 26.05 & 74.36 \\
\rowcolor{oursrow}[2.2pt][2.2pt]
\textbf{Ours} & 85/0/15/0 & 94.93 & 97.15 & 99.70 & 99.50 & 60.79 & 66.54 & 54.17 & 25.86 & 74.83 \\
\rowcolor{oursrow}[2.2pt][2.2pt]
\textbf{Ours} & 80.7/14.4/4.8/0 & 93.53 & 96.31 & 99.21 & 99.30 & 60.83 & 65.90 & 69.44 & 23.14 & 75.96 \\
\midrule
SageAttn3 & 0/100/0/0 & 93.72 & 96.63 & 99.24 & 99.29 & 60.76 & 65.89 & 65.28 & 22.53 & 75.42 \\
\rowcolor{oursrow}[2.2pt][2.2pt]
\textbf{Ours} & 13.4/64.9/21.6/0 & 93.71 & 96.55 & 99.21 & 99.28 & 61.14 & 66.73 & 66.67 & 22.70 & 75.75 \\
\midrule
\multicolumn{11}{c}{\textit{\textbf{MiniMax-H3}}} \\*
\midrule
Dense & 0/0/0/100 & 90.38 & 88.41 & 97.92 & 98.95 & 62.95 & 70.67 & 91.67 & 26.01 & 78.37 \\
XAttn & 79.9/0/0/20.1 & 89.36 & 89.23 & 98.19 & 98.66 & 62.40 & 68.47 & 91.67 & 26.30 & 78.03 \\
SVG2 & 79.9/0/0/20.1 & 89.22 & 89.24 & 98.11 & 98.64 & 60.48 & 65.31 & 88.89 & 26.29 & 77.02 \\
Sol-Attn & 80/0/0/20 & 90.33 & 88.50 & 97.89 & 98.90 & 63.08 & 70.95 & 90.28 & 26.07 & 78.25 \\
SVG-EAR & 79.9/0/0/20.1 & 89.45 & 89.16 & 98.19 & 98.78 & 60.85 & 65.52 & 87.50 & 26.29 & 76.97 \\
SVOO & 79.9/0/0/20.1 & 89.86 & 89.63 & 98.06 & 98.78 & 61.98 & 68.30 & 91.67 & 26.25 & 78.07 \\
\rowcolor{oursrow}[2.2pt][2.2pt]
\textbf{Ours} & 80/0/0/20 & 90.45 & 88.84 & 97.94 & 98.95 & 62.80 & 70.53 & 87.50 & 26.15 & 77.90 \\
\rowcolor{oursrow}[2.2pt][2.2pt]
\textbf{Ours} & 80/0/10/10 & 90.39 & 88.46 & 97.94 & 98.94 & 62.88 & 70.47 & 87.50 & 26.13 & 77.84 \\
\midrule
SpargeAttn & 80/0/20/0 & 90.36 & 89.50 & 98.00 & 98.85 & 62.64 & 69.55 & 84.72 & 26.15 & 77.47 \\
\rowcolor{oursrow}[2.2pt][2.2pt]
\textbf{Ours} & 80/0/20/0 & 90.23 & 88.88 & 97.82 & 98.82 & 62.86 & 70.16 & 90.28 & 26.13 & 78.15 \\
\rowcolor{oursrow}[2.2pt][2.2pt]
\textbf{Ours} & 74.3/19.3/6.4/0 & 90.35 & 88.19 & 98.01 & 98.91 & 62.69 & 70.10 & 87.50 & 26.15 & 77.74 \\
\midrule
SageAttn3 & 0/100/0/0 & 90.13 & 88.09 & 98.03 & 98.93 & 62.43 & 69.46 & 86.11 & 25.99 & 77.40 \\
\rowcolor{oursrow}[2.2pt][2.2pt]
\textbf{Ours} & 13.6/64.8/21.6/0 & 90.29 & 87.42 & 98.04 & 98.94 & 62.71 & 69.86 & 86.11 & 26.04 & 77.43 \\
\end{longtable}
\endgroup

%% file: tables/vbench-fewstep-results.tex
\begingroup
\captionsetup{font=small,skip=3pt}
\fontsize{8.5}{10}\selectfont
\setlength{\tabcolsep}{2.2pt}
\renewcommand{\arraystretch}{1.00}
\setlength{\heavyrulewidth}{0.6pt}
\setlength{\lightrulewidth}{0.3pt}
\arrayrulecolor{black!65}
\definecolor{oursrow}{RGB}{224,237,250}
\setlength{\LTleft}{\fill}
\setlength{\LTright}{\fill}
\setlength{\LTcapwidth}{\textwidth}
\begin{longtable}{@{}l!{\color{black!35}\vrule width 0.3pt}c!{\color{black!35}\vrule width 0.3pt}*{9}{c}@{}}
\caption{Detailed VBench scores (\%) for few-step generation. Configurations correspond to Table~\ref{tab:fewstep-results}.}\label{tab:vbench-fewstep-results}\\
\toprule
\multirow{2}{*}{Method} & \multirow{2}{*}{\makecell{Bit (\%)\\0 / 4 / 8 / 16}} & \multicolumn{9}{c}{VBench$\uparrow$} \\
\cmidrule(lr){3-11}
 & & SC & BC & TF & MS & AQ & IQ & DD & OC & AVG \\
\endfirsthead
\multicolumn{11}{c}{\tablename\ \thetable\ (continued)} \\
\toprule
\multirow{2}{*}{\textbf{Method}} & \multirow{2}{*}{\makecell{\textbf{Bit (\%)}\\0 / 4 / 8 / 16}} & \multicolumn{9}{c}{\textbf{VBench}$\uparrow$} \\
\cmidrule(lr){3-11}
 & & SC & BC & TF & MS & AQ & IQ & DD & OC & AVG \\
\midrule
\endhead
\bottomrule
\endfoot
\midrule
\multicolumn{11}{c}{\textit{\textbf{4-Step Diffusion}}} \\*
\midrule
Dense & 0/0/0/100 & 91.65 & 93.35 & 98.77 & 99.06 & 62.56 & 73.80 & 50.00 & 23.92 & 74.14 \\
Sol-Attn & 65/0/0/35 & 91.68 & 93.35 & 98.75 & 99.03 & 62.76 & 73.97 & 50.00 & 23.99 & 74.19 \\
SVG2 & 64/0/0/36 & 91.48 & 93.37 & 98.71 & 98.89 & 62.50 & 72.95 & 48.61 & 23.86 & 73.80 \\
SVG2 & 85.8/0/0/14.2 & 91.33 & 93.16 & 98.75 & 98.90 & 62.29 & 72.42 & 44.44 & 23.75 & 73.13 \\
SVG-EAR & 66.4/0/0/33.6 & 91.28 & 93.56 & 98.76 & 98.97 & 62.65 & 72.83 & 48.61 & 23.84 & 73.81 \\
SVG-EAR & 86.1/0/0/13.9 & 91.50 & 93.43 & 98.80 & 98.94 & 62.47 & 72.24 & 48.61 & 23.81 & 73.73 \\
SVOO & 65/0/0/35 & 91.54 & 93.62 & 98.73 & 98.95 & 62.33 & 72.75 & 48.61 & 23.84 & 73.80 \\
SVOO & 76.3/0/0/23.7 & 91.52 & 93.43 & 98.77 & 98.96 & 62.61 & 72.10 & 48.61 & 23.93 & 73.74 \\
\rowcolor{oursrow}[2.2pt][2.2pt]
\textbf{Ours} & 65/0/0/35 & 91.63 & 93.16 & 98.75 & 99.03 & 62.60 & 73.78 & 50.00 & 23.90 & 74.11 \\
\rowcolor{oursrow}[2.2pt][2.2pt]
\textbf{Ours} & 63.4/0/18.3/18.3 & 91.54 & 93.09 & 98.72 & 98.96 & 62.51 & 73.79 & 48.61 & 23.89 & 73.89 \\
\rowcolor{oursrow}[2.2pt][2.2pt]
\textbf{Ours} & 53.2/0/37.5/9.4 & 91.47 & 93.39 & 98.71 & 98.81 & 62.38 & 73.49 & 48.61 & 23.93 & 73.85 \\
\rowcolor{oursrow}[2.2pt][2.2pt]
\textbf{Ours} & 29.6/35.2/35.2/0 & 91.62 & 93.21 & 98.73 & 98.96 & 62.25 & 73.68 & 50.00 & 23.97 & 74.05 \\
\midrule
SpargeAttn & 65/0/35/0 & 91.62 & 93.18 & 98.75 & 99.01 & 62.53 & 73.78 & 47.22 & 23.81 & 73.74 \\
\rowcolor{oursrow}[2.2pt][2.2pt]
\textbf{Ours} & 65/0/35/0 & 91.43 & 93.45 & 98.65 & 98.45 & 61.77 & 73.07 & 48.61 & 23.77 & 73.65 \\
\rowcolor{oursrow}[2.2pt][2.2pt]
\textbf{Ours} & 54.5/34.1/11.4/0 & 91.61 & 93.16 & 98.71 & 98.93 & 62.21 & 73.57 & 50.00 & 23.82 & 74.00 \\
\midrule
SageAttn3 & 0/100/0/0 & 91.47 & 93.51 & 98.71 & 98.89 & 62.16 & 73.41 & 50.00 & 23.84 & 74.00 \\
\rowcolor{oursrow}[2.2pt][2.2pt]
\textbf{Ours} & 13.1/65.2/21.7/0 & 91.66 & 93.14 & 98.72 & 98.93 & 62.19 & 73.51 & 48.61 & 23.92 & 73.83 \\
\midrule
\multicolumn{11}{c}{\textit{\textbf{8-Step Diffusion}}} \\*
\midrule
Dense & 0/0/0/100 & 91.94 & 89.76 & 96.72 & 98.64 & 64.28 & 72.92 & 98.61 & 26.06 & 79.87 \\
Sol-Attn & 79.9/0/0/20.1 & 91.87 & 90.07 & 96.65 & 98.57 & 64.40 & 73.01 & 98.61 & 26.11 & 79.91 \\
SVG2 & 79.6/0/0/20.4 & 91.61 & 90.02 & 97.22 & 98.67 & 63.80 & 69.99 & 97.22 & 26.21 & 79.34 \\
SVG2 & 86.7/0/0/13.3 & 91.46 & 89.77 & 97.31 & 98.73 & 62.91 & 68.80 & 97.22 & 26.15 & 79.04 \\
SVG-EAR & 79.6/0/0/20.4 & 91.80 & 90.66 & 97.25 & 98.74 & 63.95 & 70.21 & 97.22 & 26.24 & 79.51 \\
SVG-EAR & 88.4/0/0/11.6 & 91.57 & 90.04 & 97.35 & 98.76 & 63.16 & 69.02 & 94.44 & 26.17 & 78.81 \\
SVOO & 80.1/0/0/19.9 & 91.67 & 90.39 & 97.33 & 98.68 & 63.26 & 69.27 & 97.22 & 26.11 & 79.24 \\
\rowcolor{oursrow}[2.2pt][2.2pt]
\textbf{Ours} & 80/0/0/20 & 91.93 & 89.79 & 96.74 & 98.64 & 64.49 & 72.71 & 98.61 & 26.11 & 79.88 \\
\rowcolor{oursrow}[2.2pt][2.2pt]
\textbf{Ours} & 74.3/0/12.8/12.8 & 91.92 & 89.64 & 96.83 & 98.60 & 64.42 & 72.73 & 98.61 & 26.08 & 79.85 \\
\rowcolor{oursrow}[2.2pt][2.2pt]
\textbf{Ours} & 57.9/0/42.1/0 & 91.71 & 90.52 & 97.08 & 98.28 & 63.66 & 71.63 & 98.61 & 26.06 & 79.69 \\
\rowcolor{oursrow}[2.2pt][2.2pt]
\textbf{Ours} & 40.5/59.5/0/0 & 91.86 & 90.36 & 97.11 & 98.66 & 64.12 & 71.07 & 98.61 & 26.03 & 79.73 \\
\rowcolor{oursrow}[2.2pt][2.2pt]
\textbf{Ours} & 50.8/24.6/24.6/0 & 91.91 & 89.94 & 96.88 & 98.59 & 64.31 & 72.34 & 98.61 & 25.95 & 79.82 \\
\midrule
SpargeAttn & 80/0/20/0 & 92.00 & 90.31 & 96.90 & 98.65 & 64.53 & 71.99 & 97.22 & 25.96 & 79.70 \\
\rowcolor{oursrow}[2.2pt][2.2pt]
\textbf{Ours} & 80/0/20/0 & 91.78 & 90.16 & 96.88 & 98.46 & 64.28 & 72.11 & 98.61 & 26.17 & 79.81 \\
\rowcolor{oursrow}[2.2pt][2.2pt]
\textbf{Ours} & 74/13/13/0 & 91.99 & 90.03 & 96.85 & 98.60 & 64.52 & 72.45 & 98.61 & 26.08 & 79.89 \\
\midrule
SageAttn3 & 0/100/0/0 & 91.81 & 90.08 & 97.02 & 98.59 & 64.08 & 71.61 & 98.61 & 25.98 & 79.72 \\
\rowcolor{oursrow}[2.2pt][2.2pt]
\textbf{Ours} & 13/65.3/21.8/0 & 91.84 & 90.21 & 97.00 & 98.58 & 64.12 & 72.03 & 98.61 & 26.00 & 79.80 \\
\end{longtable}
\endgroup

%% file: sections/appendix-visualizations.tex
\clearpage
\section{Additional Visual Comparisons}
\label{app:visual-comparisons}

Figures~\ref{fig:appendix-rail}--\ref{fig:appendix-park-cat} provide additional visual comparisons.

\begin{figure}[!ht]
\centering
\includegraphics[width=\linewidth]{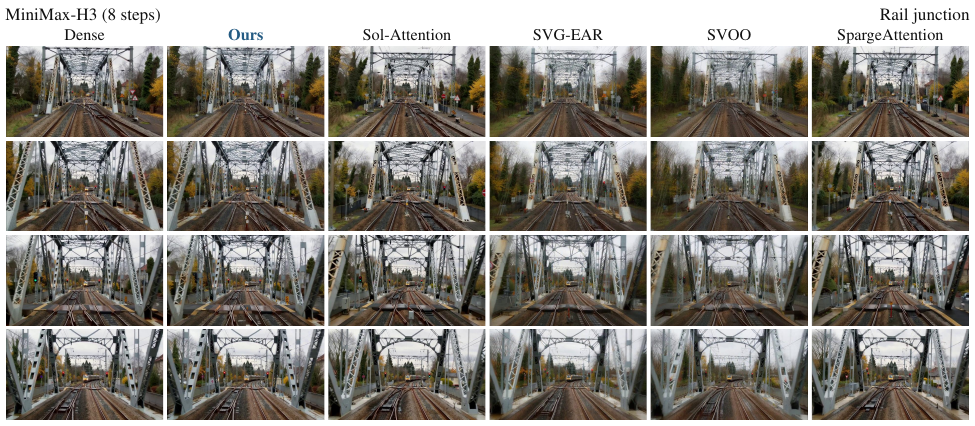}
\captionsetup{font=small,skip=3pt}
\caption{Visual comparison of rail junctions in 8-step MiniMax-H3 generation.}
\label{fig:appendix-rail}
\end{figure}

\begin{figure}[!ht]
\centering
\includegraphics[width=\linewidth]{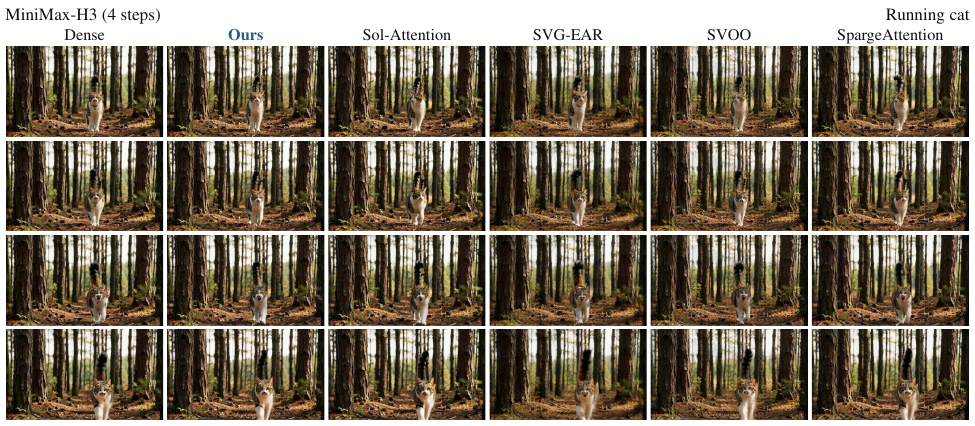}
\captionsetup{font=small,skip=3pt}
\caption{Visual comparison of a running cat in 4-step MiniMax-H3 generation.}
\label{fig:appendix-running-cat}
\end{figure}

\begin{figure}[!ht]
\centering
\includegraphics[width=\linewidth]{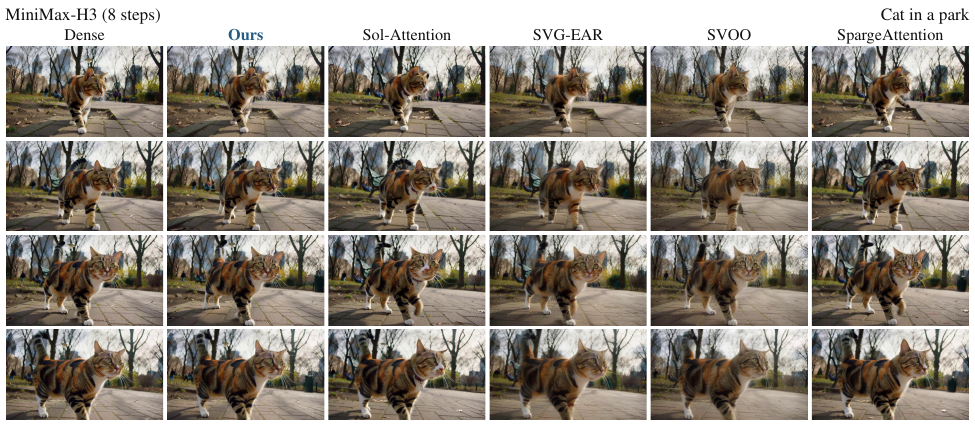}
\captionsetup{font=small,skip=3pt}
\caption{Visual comparison of a cat in a park in 8-step MiniMax-H3 generation.}
\label{fig:appendix-park-cat}
\end{figure}